\documentclass[journal]{IEEEtran}

\usepackage{graphicx}
\usepackage[percent]{overpic}
\usepackage{subfig}
\usepackage{amsmath}
\usepackage{amsthm} 
\usepackage{amssymb}
\usepackage[font=footnotesize]{caption} 
\usepackage{subfig} 
\usepackage[noadjust]{cite} 
\usepackage{color}
\usepackage{algorithm} 
\usepackage{algpseudocode} 
\usepackage{enumerate}
\usepackage[left=0.75in, right=0.75in, top=0.75in, bottom=0.75in]{geometry}
\usepackage{authblk} 
\usepackage{arydshln} 

\usepackage{bm}
\usepackage{tikz}
\usetikzlibrary{calc} 
\usetikzlibrary{shapes} 
\usetikzlibrary{chains}
\usetikzlibrary{fit}
\usetikzlibrary{arrows}
\usetikzlibrary{decorations.text} 
\usetikzlibrary{decorations.markings}
\usetikzlibrary{decorations.pathmorphing} 
\usetikzlibrary{shadows}
\usetikzlibrary{patterns}
\usetikzlibrary{matrix}
\usepackage{pgfplots}
\usepackage[europeanresistors]{circuitikz}
\usepackage[outline]{contour} 
\contourlength{1.5pt}

\newtheorem{assumption}{Assumption}
\newtheorem{remark}{Remark}
\newtheorem{proposition}{Proposition}

\graphicspath{{figures/}}

\usepackage{times}

\usepackage[numbers]{natbib}
\usepackage{multicol}
\usepackage[bookmarks=true]{hyperref}

\usepackage{xcolor}
\usepackage{colortbl} 

\usepackage{graphicx}
\usepackage{caption}
\usepackage{verbatim}
\usepackage{graphicx}

\begin{document}

\title{DiffOG: Differentiable Policy Trajectory Optimization with Generalizability}

\author{Zhengtong Xu, Zichen Miao, Qiang Qiu, Zhe Zhang, Yu She$^{*}$

\thanks{$^{*}$Address all correspondence to this author.}
\thanks{Zhengtong, Zichen, Qiang, Zhe, and Yu are with Purdue University, West Lafayette, USA  
{\tt\footnotesize xu1703, miaoz, qqiu, zhan5111, shey@purdue.edu}}%
\thanks{This work was supported in part by the National Science Foundation (NSF) under Awards 2322056, 2423068, and 2520136, and in part by the United States
Department of Agriculture (USDA) under Awards 2023-67021-39072 and 2024-67021-42878. }
}

\let\oldtwocolumn\twocolumn
\renewcommand\twocolumn[1][]{%
    \oldtwocolumn[{#1}{
    \begin{center}
    \includegraphics[trim=20 0 0 0,clip,width=1.0\textwidth]{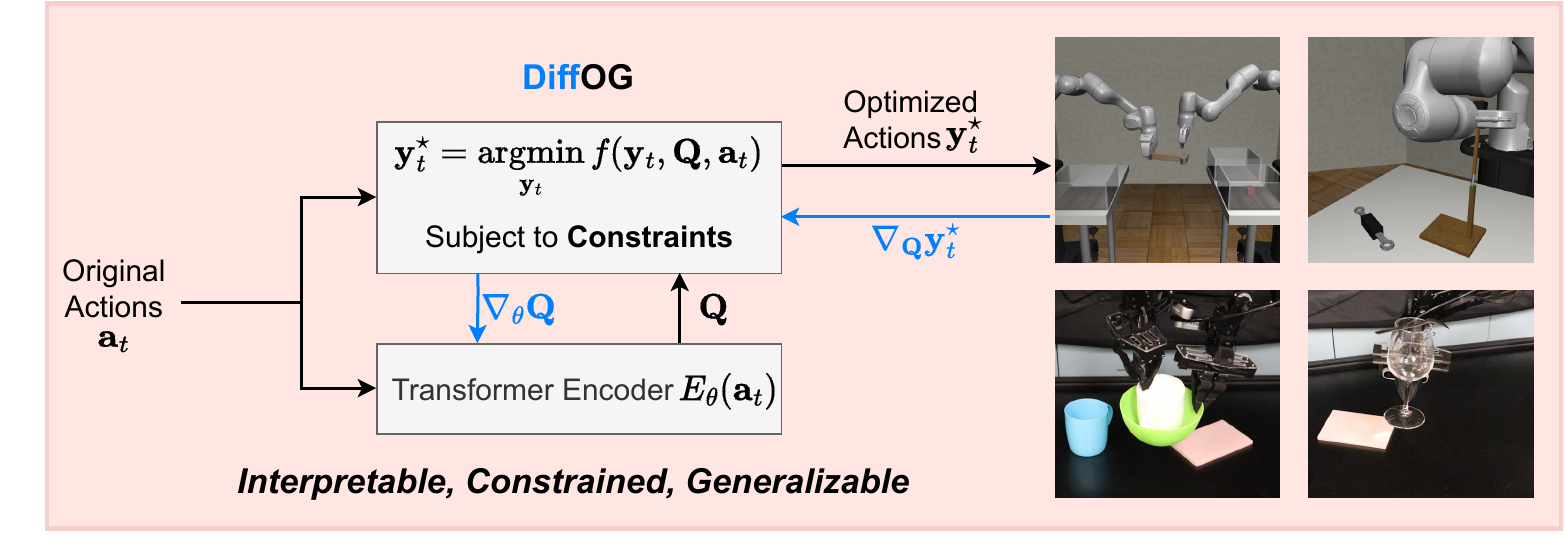}
           \captionof{figure}{We introduce differentiable policy trajectory optimization with generalizability (DiffOG). Visuomotor policies enhanced by DiffOG generate smoother, constraint-compliant action trajectories in a more interpretable way. DiffOG introduces a novel transformer-based differentiable trajectory optimization framework tailored for action refinement in imitation learning. Leveraging the differentiability of the optimization layer and the high capacity of the transformer, DiffOG can be trained on demonstration data to adapt to the diverse characteristics of trajectories across different tasks. We evaluate DiffOG across 13 tasks and showcase four representative ones here. These selected tasks present several key challenges, including long-horizon dual-arm manipulation, high-precision control, and smooth, constraint-satisfying trajectory generation.}
           \label{fig:firstPage}
        \end{center}
    }]
}

\maketitle

\begin{abstract}
Imitation learning-based visuomotor policies excel at manipulation tasks but often produce suboptimal action trajectories compared to model-based methods. Directly mapping camera data to actions via neural networks can result in jerky motions and difficulties in meeting critical constraints, compromising safety and robustness in real-world deployment. For tasks that require high robustness or strict adherence to constraints, ensuring trajectory quality is crucial. However, the lack of interpretability in neural networks makes it challenging to generate constraint-compliant actions in a controlled manner. This paper introduces differentiable policy trajectory optimization with generalizability (DiffOG), a learning-based trajectory optimization framework designed to enhance visuomotor policies. By leveraging the proposed differentiable formulation of trajectory optimization with transformer, DiffOG seamlessly integrates policies with a generalizable optimization layer. DiffOG refines action trajectories to be smoother and more constraint-compliant while maintaining alignment with the original demonstration distribution, thus avoiding degradation in policy performance. We evaluated DiffOG across 11 simulated tasks and 2 real-world tasks. The results demonstrate that DiffOG significantly enhances the trajectory quality of visuomotor policies while having minimal impact on policy performance, outperforming trajectory processing baselines such as greedy constraint clipping and penalty-based trajectory optimization. Furthermore, DiffOG achieves superior performance compared to existing constrained visuomotor policy. For more details, please visit the project website: https://zhengtongxu.github.io/diffog-website/.
\end{abstract}
\IEEEpeerreviewmaketitle

\section{Introduction}

Imitation learning \cite{pomerleau1988alvinn} has emerged as a popular paradigm for endowing robots with complex manipulation skills by leveraging human demonstrations. By formulating a supervised learning problem that maps sensor observations directly to actions via neural networks, imitation learning-based visuomotor policies have demonstrated effectiveness across a broad spectrum of tasks \cite{chi2023diffusionpolicy,chi2024universal,zhao2023learning,fu2024mobile,shafiullah2023bringing}. Yet, these learned policies often exhibit notable limitations when applied to real-world settings, where safety, robustness, and strict constraint satisfaction are paramount. In contrast to classical model-based trajectory optimization methods which explicitly account for motion constraints and can ensure smooth and reliable trajectories \cite{schulman2014motion}, imitation learning approaches may produce suboptimal or jerky actions. Such trajectories can reduce the robustness of the whole robotic systems and raise critical safety concerns.

Meanwhile, various challenges in imitation learning, such as the representation of multi-modal action distributions \cite{chi2023diffusionpolicy,lee2024behavior} and training large, generalized policies on extensive datasets \cite{octo_2023,o2023open,black2024pi_0}, have been the focus of ongoing research. Despite these advancements, ensuring the generation of high-quality and constraint-compliant trajectories for neural network-based policies remains a significant and open challenge. Unlike traditional model-based motion generation methods, learning-based approaches tend to be less controlled and more difficult to interpret. For instance, learning-based policies lack mechanisms to directly control trajectory smoothness or impose motion constraints through tunable parameters.

Instead, current approaches to improve trajectory quality are mainly based on post-hoc processing, such as applying trajectory optimization methods to refine the outputs of the learned policy. However, because constraints or smoothness objectives are not inherently considered during the policy training process, these adjustments can result in trajectories that deviate from those in the demonstration dataset. This misalignment can ultimately degrade the policy's performance. This often makes such post-hoc processing methods difficult to generalize across a wide range of tasks, particularly in scenarios involving high-dimensional action spaces, such as dual-arm manipulation, or long-horizon tasks with complex and highly varied trajectories.

In this paper, we introduce differentiable policy trajectory optimization with generalizability (DiffOG), a generalizable framework that seamlessly integrates transformer-based trajectory optimization with imitation learning to produce smooth, constraint-compliant action trajectories for robotic manipulation, as shown in Fig.~\ref{fig:firstPage}.

The core contributions of DiffOG can be summarized as follows:

1.  We propose a differentiable trajectory optimization framework designed to refine the actions generated by robot policies. The optimized trajectories are capable of accomplishing the demonstrated tasks while satisfying hard constraints and exhibiting improved smoothness. Built upon a supervised learning paradigm, this framework offers high flexibility and generalizes well across diverse tasks and policy architectures. Compared to traditional approaches that enforce constraints through post-hoc processing, which often lead to actions deviating from the demonstration distribution, our method improves trajectory quality while maintaining alignment with the demonstrated behavior, thereby better preserving policy performance.

2. We integrate a transformer-based trajectory encoder into the proposed differentiable trajectory optimization framework, which substantially enhances its representational capacity. This design allows DiffOG to adapt trajectory optimization to the specific characteristics of diverse tasks, including long-horizon dual-arm manipulation, with strong generalization ability. Furthermore, through a series of deliberate design choices and theoretical analyses, our transformer-based differentiable optimization layer aligns with interpretable theoretical principles while ensuring feasibility and stable training. By combining model-based formulations with neural networks, this approach fully exploits differentiability to improve both the effectiveness and flexibility of DiffOG.

3. Through rigorous analysis, we establish the interpretability of DiffOG under this training paradigm and elucidate the underlying principles that guide its operation. Extensive validation in both simulation and real-world settings demonstrates the broad effectiveness of DiffOG, as well as the flexibility and generalization capabilities of its transformer-based trajectory optimization. Overall, DiffOG provides a practical and scalable example of integrating visual policies with a rigorous and interpretable optimization layer.

 We evaluated DiffOG across 11 simulated tasks and 2 real-world tasks. The results demonstrate that DiffOG significantly enhances the trajectory quality of visuomotor policies with generalizability while having minimal impact on policy performance, outperforming trajectory processing baselines such as constraint clipping and penalty-based trajectory optimization. Furthermore, DiffOG achieves superior performance compared to existing constrained visuomotor policy \cite{xu2024leto}.

\section{Related Work}

\subsection{Supervised Policy Learning}

Supervised policy learning entails training a policy using supervised learning techniques on a dataset of pre-collected demonstrations \cite{pomerleau1988alvinn}. Recent advancements in machine learning algorithms and data collection frameworks have substantially advanced the capabilities and effectiveness of supervised policy learning in manipulation applications. For example, various data collection frameworks and devices have been proposed in \cite{chi2024universal, zhao2023learning, fu2024mobile, shafiullah2023bringing, seo2024legato, wang2024dexcap}, enabling efficient data collection and the training of effective manipulation policies. To better capture the multi-modal action distributions present in demonstration data, prior work have explored various techniques, such as using diffusion models \cite{chi2023diffusionpolicy}, VQ-VAE \cite{lee2024behavior}, and energy-based models \cite{florence2022implicit}. Furthermore, various effective action representations for supervised policy learning have been proposed, such as flow by point tracking \cite{xu2024flow,wen2023any} and equivariant action representation \cite{wangequivariant}. Moreover,  {many work} have contributed to training more general and  {scalable policies} on large-scale datasets \cite{octo_2023,o2023open,black2024pi_0,liu2024rdt}.

However, the aforementioned methods largely overlook the importance of incorporating constraints into the actions generated by policy outputs. Despite these advancements, ensuring the generation of high-quality and constraint-compliant trajectories for neural network-based policies remains a significant and open challenge. DiffOG addresses this issue by applying transformer-based differentiable trajectory optimization, refining its generated actions to improve smoothness, adherence to hard constraints, and fidelity to the demonstration.

\subsection{Policies with Differentiable Optimization}

The integration of differentiable optimization into policy learning has garnered significant attention due to its potential to combine model-based formulations with the representational power of neural networks. Previous studies have investigated the integration of model-based structures with policies through differentiable optimization, including differentiable model predictive control (MPC) \cite{amos2018differentiable} and the Koopman operator framework \cite{retchin2023koopman}. However, these methods are predominantly designed for low-dimensional observations, which limits their flexibility and applicability to the general formulation of visuomotor policies. The work in \cite{xu2024letac} presents a tactile-reactive grasping controller that integrates an image encoder with differentiable MPC. Although effective for specific tactile-based grasping tasks, this approach lacks the formulation required for broader policy learning frameworks. In the domain of obstacle avoidance, an end-to-end learning framework incorporating differentiable optimization is proposed in \cite{xiao2021barriernet}. However, this approach is specifically designed for navigation tasks and lacks applicability to robotic manipulation scenarios and visuomotor policies.

Riemannian motion policies \cite{ratliff2018Riemannian} and their extensions \cite{li2021rmp2, cheng2020rmp,rana2021towards} advance the generation of motions in high-dimensional spaces with complex, nonlinear dynamics into end-to-end robot learning by leveraging automatic differentiation and Riemannian geometry. However, in the context of policy learning, Riemannian motion policies are primarily designed for tasks involving low-dimensional observations, such as those demonstrated through kinesthetic teaching \cite{rana2021towards}. In contrast, DiffOG is tailored for a broader range of visuomotor policies, refining policies that utilize camera observations as input. DiffOG focuses on generating task-completing trajectories while enhancing smoothness and ensuring compliance with hard constraints.

DiffTORI \cite{wan2024difftori} integrates differentiable trajectory optimization into the policy representation to generate actions, with a focus on capturing multi-modal action distributions in the context of imitation learning. However, DiffTORI lacks explicit and rigorous formulations for constraints and smoothness, limiting its ability to optimize trajectory smoothness and enforce compliance with hard constraints. In contrast, DiffOG incorporates a rigorously designed optimization layer that not only enables actions to satisfy hard constraints, but also enhances the interpretability of the policy.

Leto \cite{xu2024leto}, by employing differentiable optimization, enables end-to-end training and inference that embedded trajectory optimization, allowing constraints to be incorporated as part of the training objectives. However, Leto was designed as a discriminative model, which limits its ability to represent multi-modal action distributions \cite{florence2022implicit}. Its use of an optimization layer as the action head makes it difficult to integrate with commonly-used diffusion-based action heads \cite{chi2023diffusionpolicy,octo_2023,liu2024rdt,ze20243d}, significantly reducing its flexibility and applicability. In comparison, DiffOG offers greater flexibility and enhanced representational capacity, enabling it to achieve superior policy performance and adapt more effectively to a wider range of tasks.

\section{Method}

In this section, the details of the DiffOG will be introduced.

\subsection{Overview}

In imitation learning, the policy learns the demonstrated task skills through supervised learning on the demonstration dataset, as shown in Fig.~\ref{fig:overview}. The successful execution of a task depends on the policy generating action data that matches the distribution of the demonstration dataset. The objective of supervised learning in this context is to enable the policy to  {generate actions that follow} the same distribution as the demonstration dataset. However, traditional policy learning neglects trajectory constraints and smoothness. As a result, post-processing the policy’s output to enforce these properties tends to shift the actions away from the demonstrated distribution, thereby compromising policy performance.

To overcome this challenge, DiffOG leverages the differentiability of the optimization problem to directly learn a trajectory optimizer within a supervised learning framework. This trajectory optimizer is designed to balance objectives including constraint satisfaction, smoothness, and fidelity to the dataset, such that the resulting actions remain within the demonstrated distribution. Moreover,  {by incorporating a transformer} into the differentiable optimization layer, DiffOG significantly enhances both the generalization ability and representational power of trajectory optimization, making it applicable to a wide range of tasks, including long-horizon manipulation tasks with high-dimensional action spaces.

The primary objective of DiffOG is to take an unoptimized action trajectory ${\mathbf{a}}_t$ as input and output an optimized trajectory that can accomplishes the demonstrated task, ensuring constraint satisfaction and improving smoothness. During training, DiffOG supports two modes:

1. DiffOG dataset training: Supervised learning is performed directly on the demonstration dataset.

2. DiffOG refine training: Supervised learning refines the actions generated by a pre-trained policy.

In the inference phase, DiffOG is applied atop the base policy to optimize the actions. The overall workflow is illustrated in Figs.~\ref{fig:firstPage} and \ref{fig:overview}, showing how an unoptimized trajectory is transformed into an optimized one through our transformer-based differentiable trajectory optimization framework. In the following sections, we elaborate on the process of transforming an unoptimized action sequence into an optimized one, explain why it is trainable, describe how this process is designed to be generalizable, and detail the training procedure.

\subsection{Action Space}

We first analyze the action space to introduce the trajectory formulation of the actions generated by visuomotor policies.

The datasets used in imitation learning adhere to the format of
$$\mathcal{D}=\left\{\left(o_0^i, a_0^i, o_1^i,a_1^i, \ldots, o_{T^i}^i,a_{T^i}^i\right)\right\}_{i=1}^{N_d},$$ where $o$ represents the observations of the robot, such as images of the cameras with different views and robot states, and $a\in\mathbb{R}^{D_a}$ represents the demonstrated action aligned with the observation $o$. $N_d$ is the total number of demonstrated trajectories, $T$ is the total length of a trajectory, and $D_a$ is the dimension of action $a$.

 {Actions fall into two categories}: discrete actions, such as controlling grasping and releasing using discrete values \cite{mandlekar2022matters}, and continuous actions, such as the motion of the robot. Discrete actions, by their nature, cannot form a continuous trajectory and therefore do not require trajectory optimization. To address this distinction, we introduce a selection matrix $S \in \mathbb{R}^{D_c \times D_a}$ that identifies the action dimensions to be considered for trajectory optimization. Specifically, $S$ is used to select: 1) degrees of freedoms (DOFs) that are inherently continuous and capable of forming a trajectory, and 2) DOFs that require trajectory optimization in practical applications. The resulting filtered action variable is expressed as $c = Sa \in \mathbb{R}^{D_c}$. In this paper, we filter out the grasping action using $S$, as grasping typically does not require trajectory motion constraints and optimization.

Recently, considerable papers demonstrate that using the training objective of predicting a sequence of actions results in a more effective policy compared to predicting a single action \cite{chi2023diffusionpolicy,zhao2023learning,lee2024behavior}. Therefore, at time step $t$, the action sequence output of policies adhere to the format of 
\begin{align}
{\mathbf{a}}_t = \left[
{a}_t^{\mathrm{T}},{a}_{t+1}^{\mathrm{T}},\ldots,{a}_{t+T_p-1}^{\mathrm{T}}
\right]^{\mathrm{T}}\in \mathbb{R}^{T_pD_a} \label{eq:action_eq}
,\end{align}
where $T_p$ is the length of the predicted action sequence and $\mathbf{a}_t$ is a sequence of actions $a$. Apply the selection matrix $S \in \mathbb{R}^{D_c \times D_a}$ to $\mathbf{a}_t$, we have
\begin{align}
{\mathbf{c}}_t = \mathbf{S}{\mathbf{a}}_t = \left[
{c}_t^{\mathrm{T}},{c}_{t+1}^{\mathrm{T}},\ldots,{c}_{t+T_p-1}^{\mathrm{T}}
\right]^{\mathrm{T}} \in \mathbb{R}^{T_pD_c},\label{eq:select}
\end{align}
where $\mathbf{S} = \text{blkdiag}\left(S,\dots,S \right)
    \in \mathbb{R}^{T_pD_c \times T_pD_a}$ is a sequence of the selection matrix $S$ and $\mathbf{c}_t$ is a sequence of filtered actions $c$.
    
In discrete trajectory form \eqref{eq:select}, the time derivative of ${\mathbf{c}}_t$ is expressed as
\begin{align}
\frac{d{\mathbf{c}}_t}{dt} &=\left[(\frac{d{{c}}_t}{dt})^{\mathrm{T}},(\frac{d{{c}}_{t+1}}{dt})^{\mathrm{T}},\ldots,(\frac{d{{c}}_{t+T_p-2}}{dt})^{\mathrm{T}}\right]\notag\\
 & = \frac{1}{\Delta t}\left[{c}_{t+1}^{\mathrm{T}}-{c}_{t}^{\mathrm{T}},\ldots,{c}_{t+T_p-1}^{\mathrm{T}} - {c}_{t+T_p-2}^{\mathrm{T}}\right]^{\mathrm{T}}\notag\\
 & =  \frac{1}{\Delta t}\left[ \begin{matrix}
-1 & 1 &  &  &  \\
 & -1 & 1 &  &  \\
 &  & \ddots & \ddots &  \\
 &  &  & -1 & 1 
\end{matrix}  \right]{\mathbf{c}}_t\notag\\
 & = \mathbf{A}_{\text{diff}} {\mathbf{c}}_t \in \mathbb{R}^{T_pD_c-D_c}.\label{eq:dc}
 \end{align}

\begin{figure*}[t]
\centering
\begin{overpic}[trim=0 0 40 0,clip, width=0.98\textwidth]{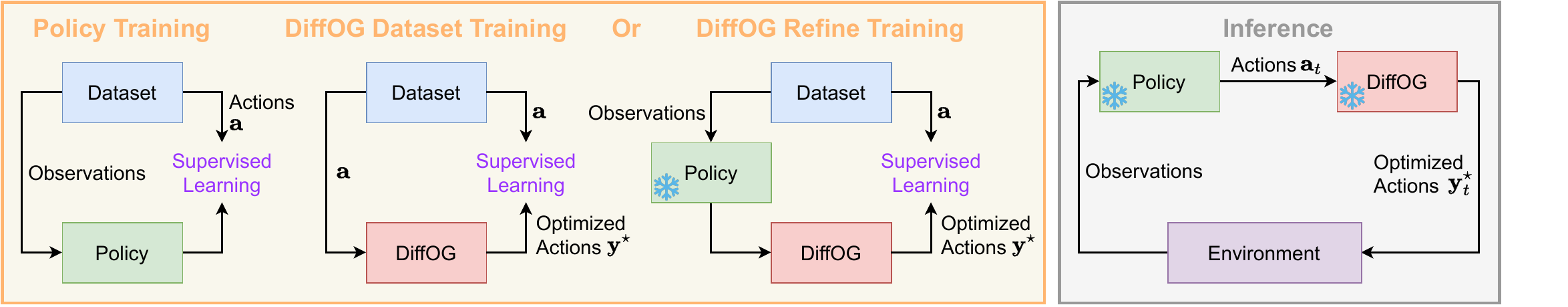}
\end{overpic}
\caption{High-level Overview of training and inference of DiffOG. A detailed illustration of DiffOG (red block) can be found in Fig. \ref{fig:firstPage}.}
\label{fig:overview}
\end{figure*}

\subsection{Differentiable Trajectory Optimization}\label{sec:dto}

In this section, we introduce the formulation of the differentiable trajectory optimization problem, which is the optimization process that converts an unoptimized action trajectory to the optimized action trajectory. We first define the optimized trajectory output by DiffOG as:

$${\mathbf{y}}_t = \left[
\hat{a}_t^{\mathrm{T}},\hat{a}_{t+1}^{\mathrm{T}},\ldots,\hat{a}_{t+T_p-1}^{\mathrm{T}}
\right]^{\mathrm{T}}\in \mathbb{R}^{T_pD_a}.
$$
The definition of ${\mathbf{y}}_t$ follows \eqref{eq:action_eq} and~$\hat{a}$~here simply represents optimized actions. Similarly, from equations \eqref{eq:select} and \eqref{eq:dc}, the following equations of ${\mathbf{y}}_t$ can be derived
\begin{align}
{\mathbf{\hat{c}}}_t &= \mathbf{S}{\mathbf{y}}_t = \left[
\hat{c}_t^{\mathrm{T}},\hat{c}_{t+1}^{\mathrm{T}},\ldots,\hat{c}_{t+T_p-1}^{\mathrm{T}}
\right]^{\mathrm{T}} \in \mathbb{R}^{T_pD_c},\label{eq:select_y}\\
\frac{d\hat{\mathbf{c}}_t}{dt} &=  \mathbf{A}_{\text{diff}}\mathbf{S}{\mathbf{y}}_t. \label{eq:dc_y}
\end{align}

Based on all notations we defined, we can write the following optimization problem which take $\mathbf{a}_t$ as input and outputs the optimized $\mathbf{y}^{\star}_t$
 \begin{align}
\mathbf{y}^{\star}_t=&\underset{\mathbf{y}_t}{\operatorname{argmin}} \frac{1}{2} \mathbf{y}_t^{\mathrm{T}} \mathbf{Q} \mathbf{y}_t+\mathbf{a}_t^{\mathrm{T}} \mathbf{y}_t +\frac{\alpha}{2} (\frac{d\hat{\mathbf{c}}_t}{dt}) ^{\mathrm{T}} \frac{d\hat{\mathbf{c}}_t}{dt}, \label{eq:opt1}\\
\text { subject to }~
 &d_{\text{min}} \Delta t \leq \hat{c}_{t+k+1} - \hat{c}_{t+k} \leq d_{\text{max}} \Delta t, \label{eq:vel_con}\\
 &k = 0, 1,  \ldots, T_p-2.\notag
\end{align}
In optimization problem \eqref{eq:opt1}, $\mathbf{Q}$ is a weight matrix, $\alpha$ is a scalar, and $d_{\text{max}}, d_{\text{min}}$ are constraint bounds for the derivatives of the robot actions. Intuitively, $\alpha$ reflects the degree to which the trajectory smoothness is emphasized. When the sum of the derivatives of the robot actions $\hat{\mathbf{c}}_t$ in the generated trajectory $\mathbf{y}_t$ is large, it indicates that the trajectory is less smooth. For instance, excessive acceleration or velocity contributes to an increase in the cost function. Therefore, we call $\alpha$ as smoothing weight. 
Optimization problem \eqref{eq:opt1} aims to solve for an action trajectory that satisfies the constraints on action derivatives while achieving greater smoothness, making it a trajectory optimization problem. $\mathbf{Q}$ also has a intuitive theoretical interpretation, which will be elaborated in detail in Section~\ref{sec:obj}. In Section~\ref{sec:obj}, we will  {dive} deeper into the trajectory optimization objectives by incorporating learning-based perspectives.

Optimization problem \eqref{eq:opt1} is a quadratic program. In DiffOG, we instantiate this class by proposing a specific, learnable trajectory optimization formulation tailored for imitation learning. Thus far, we have introduced the overall formulation of trajectory optimization in DiffOG. To further enable learning within this optimization layer, it is essential to emphasize a key property: differentiability. A differentiable trajectory optimization module can be seamlessly integrated with neural networks for end-to-end training, enabling the optimization process to acquire stronger representational capacity through data-driven learning. In what follows, we  {dive} into the details of differentiability, beginning with the following proposition:
\begin{proposition}
If $d_{\text{min}} < d_{\text{max}}$ and $\mathbf{Q}$ is symmetric positive definite, optimization problem \eqref{eq:opt1} is always feasible and strictly convex.\label{prop:feasibility}
\end{proposition}
\begin{proof}
The cost function of optimization problem \eqref{eq:opt1} can be expressed as 
$$
\frac{1}{2} \mathbf{y}_t^{\mathrm{T}} \left(\mathbf{Q} + \alpha \mathbf{S}^\mathrm{T} \mathbf{A}_{\text{diff}}^\mathrm{T} \mathbf{A}_{\text{diff}} \mathbf{S}\right) \mathbf{y}_t + \mathbf{a}_t^{\mathrm{T}} \mathbf{y}_t.
$$
Because rank($S$) = $D_c$, rank($\mathbf{S}$) = $T_pD_c$. Because $c$ is taken from a subset of $a$, it follows that $D_a \geq D_c$. Therefore, $\mathbf{S}$ has full rank. Since \(\mathbf{Q}\) is symmetric positive definite and the smoothing weight $\alpha \geq 0$, the matrix \(\mathbf{Q} + \alpha \mathbf{S}^\mathrm{T} \mathbf{A}_{\text{diff}}^\mathrm{T} \mathbf{A}_{\text{diff}} \mathbf{S}\) is also symmetric positive definite. The condition \(d_{\text{min}} < d_{\text{max}}\) ensures the feasibility of constraints, guaranteeing constraints are not contradictory. 

Therefore, optimization problem \eqref{eq:opt1} is always feasible and strictly convex.
\end{proof}
\begin{remark}[Trajectory Optimization]
Proposition~\ref{prop:feasibility} guarantees that optimization problem \eqref{eq:opt1} is feasible and strictly convex, ensuring the existence of a unique optimal solution $\mathbf{y}^{\star}_t$. Moreover, the smoothing term $(\frac{d\hat{\mathbf{c}}_t}{dt})^{\mathrm{T}} \frac{d\hat{\mathbf{c}}_t}{dt}$ represents a physically interpretable quantity which reflects the trajectory's smoothness. For instance, a large $(\frac{d\hat{\mathbf{c}}_t}{dt})^{\mathrm{T}} \frac{d\hat{\mathbf{c}}_t}{dt}$ indicates a lack of smoothness in the trajectory. Consequently, optimization problem \eqref{eq:opt1} can enforce trajectory smoothing by minimizing the associated cost. Finally, the condition $d_{\text{min}} < d_{\text{max}}$ carries explicit physical meaning. For example, if the action space is defined in terms of joint angles, this condition ensures that the minimum and maximum joint velocity constraints do not conflict. 
\end{remark}

In terms of differentiability, by Proposition~\ref{prop:feasibility} and the work in \cite{amos2017optnet}, we can get the following remark.
\begin{remark}[Differentiability]
Since $\mathbf{Q} \succ 0$, optimization problem \eqref{eq:opt1} is strictly convex, ensuring a unique optimal solution $\mathbf{y}^{\star}_t$ that is continuous and subdifferentiable with respect to all variables everywhere, and differentiable at all
but a measure-zero set of points. In this case, for example, the gradient $\nabla_{\mathbf{Q}} \mathbf{y}^{\star}_t$ can be computed explicitly via the solution of the KKT conditions, leveraging the fact that the optimality conditions are differentiable \cite{amos2017optnet}.
\label{rem:diff}
\end{remark}

\subsection{The Learning Objectives of DiffOG}\label{sec:obj}

According to Remark~\ref{rem:diff}, we know that optimization problem \eqref{eq:opt1} is theoretically differentiable with respect to all variables. However, in practice, we can design the optimization process such that certain variables remain differentiable while others are treated as constants. This section will introduce the rationale behind our design choice of constants and differentiable variables and provide details on the resulting learning objectives. 

First, let us assume that $\mathbf{Q}$ is a symmetric positive definite matrix. According to Proposition~\ref{prop:feasibility}, this condition with proper constraints values ensures that optimization problem \eqref{eq:opt1} is both feasible and differentiable.
By Cholesky factorization, there exists a unique upper-triangular matrix $\mathbf{R}$ with positive diagonal entries such that 
$\mathbf{Q} = \mathbf{R}^\mathrm{T} \mathbf{R}$. 
Because $\mathbf{Q}$ is symmetric positive definite, the factor $\mathbf{R}$ is invertible (i.e., full rank). 

Next, suppose we define the vector $\mathbf{g}_t$ by the relation 
$$\mathbf{g}_t = -  (\mathbf{R}^\mathrm{T})^{-1}  \mathbf{a}_t.$$
Since $\mathbf{R}$ is full rank and upper triangular, its transpose $\mathbf{R}^\mathrm{T}$ is also invertible. Hence $(\mathbf{R}^\mathrm{T})^{-1}$ exists, and the vector $\mathbf{g}_t$ is uniquely determined by the above equation. In addition, $\mathbf{a}_t = -  \mathbf{R}^{\mathrm{T}}  \mathbf{g}_t $. Then, part of the cost function of optimization problem \eqref{eq:opt1} can be rewritten as
\begin{align*}
\frac{1}{2} \mathbf{y}_t^{\mathrm{T}} \mathbf{Q}  \mathbf{y}_t  &+  \mathbf{a}_t^{\mathrm{T}} \mathbf{y}_t\\ 
&= \frac{1}{2} \mathbf{y}_t^{\mathrm{T}} \mathbf{R}^{\mathrm{T}} \mathbf{R}  \mathbf{y}_t  +  \mathbf{a}_t^{\mathrm{T}} \mathbf{y}_t 
\\
&= \frac{1}{2} \mathbf{y}_t^{\mathrm{T}} \mathbf{R}^{\mathrm{T}} \mathbf{R}  \mathbf{y}_t 
    -  \mathbf{g}_t^{\mathrm{T}}\mathbf{R}  \mathbf{y}_t 
\\
&= \frac{1}{2}
   \Bigl(\mathbf{y}_t^{\mathrm{T}} \mathbf{R}^{\mathrm{T}} \mathbf{R}  \mathbf{y}_t
          -  2  \mathbf{g}_t^{\mathrm{T}}\mathbf{R}  \mathbf{y}_t
          +  \mathbf{g}_t^{\mathrm{T}}\mathbf{g}_t
          -  \mathbf{g}_t^{\mathrm{T}}\mathbf{g}_t
   \Bigr)
\\
&= \frac{1}{2}  \Bigl\|\mathbf{R}\mathbf{y}_t - \mathbf{g}_t\Bigr\|^2 
    +  \underbrace{\Bigl(-  \tfrac{1}{2}  \|\mathbf{g}_t\|^2\Bigr)}_{\text{independent of \(\mathbf{y}_t\)}}.
\end{align*}

Since the term \(\bigl(-\tfrac{1}{2}\|\mathbf{g}_t\|^2\bigr)\) is independent of \(\mathbf{y}_t\), it does not affect the optimization process and can be omitted from the cost function of optimization problem~\eqref{eq:opt1}. Thus, optimization problem~\eqref{eq:opt1} can be rewritten as:
 \begin{align}
\mathbf{y}^{\star}_t=&\underset{\mathbf{y}_t}{\operatorname{argmin}} \frac{1}{2}  \bigl\|\mathbf{R} \mathbf{y}_t - \mathbf{g}_t\bigr\|^2 +\frac{\alpha}{2} (\frac{d\hat{\mathbf{c}}_t}{dt}) ^{\mathrm{T}} \frac{d\hat{\mathbf{c}}_t}{dt},\label{eq:opt2}\\
\text { subject to }~
 &d_{\text{min}} \Delta t \leq \hat{c}_{t+k+1} - \hat{c}_{t+k} \leq d_{\text{max}} \Delta t,\notag\\
 &k = 0, 1,  \ldots, T_p-2.\notag
\end{align}

We can observe that if we dropout the constraints and smoothing term, the solution of optimization problem \eqref{eq:opt2} will be
\begin{align}
   \mathbf{y}_t 
 = 
-\mathbf{R}^{-1}  (\mathbf{R}^{\mathrm{T}})^{-1}  \mathbf{a}_t
 = 
-(\mathbf{R}^{\mathrm{T}}\mathbf{R})^{-1}  \mathbf{a}_t 
 = 
-  \mathbf{Q}^{-1}  \mathbf{a}_t, \label{eq:transformation}
\end{align}
which represent a linear transformation of the original action $\mathbf{a}_t$. In the presence of constraints and smoothing terms, optimization problem \eqref{eq:opt2} is no longer solely about a linear transformation; instead, it simultaneously accounts for three aspects: transformation, constraints, and smoothness. Next, we demonstrate how differentiability and learning can be leveraged to integrate these three aspects into three interpretable learning objectives.

As previously mentioned, optimization problem \eqref{eq:opt2} is derived by rewriting optimization problem \eqref{eq:opt1} using Cholesky factorization $\mathbf{Q} = \mathbf{R}^\mathrm{T} \mathbf{R}$. Therefore, in optimization problem \eqref{eq:opt2}, we can make $\mathbf{Q}$ learnable/differentiable and train differentiable trajectory optimization \eqref{eq:opt2} with the loss function
\begin{align}
\mathcal{L} = \mathbb{E}_{\mathbf{a}_t \in \mathcal{D}}\| \mathbf{y}_t^{\star}  - \mathbf{a}_t\|^2. \label{eq:loss}
\end{align}
In this case, 1) the forward pass of DiffOG involves solving optimization problem \eqref{eq:opt2} to generate $\mathbf{y}^{\star}_t$, and 2) the backpropagation process updates $\mathbf{Q}$ using the loss function \eqref{eq:loss}. 

In both the DiffOG dataset training and DiffOG refine training, the differentiability of optimization problem \eqref{eq:opt1} allows gradients from the loss to be propagated to $\mathbf{Q}$, as shown in Fig.~\ref{fig:firstPage}. This enables the update of the optimization problem itself, making the inference process progressively adapt to the demonstration data. The interpretable learning objectives can then be summarized in the following remark.

\begin{remark}[Interpretability]
Training differentiable trajectory optimization \eqref{eq:opt2} with loss function \eqref{eq:loss} results in three interpretable learning objectives:

1)
The first objective is to make the output $\mathbf{y}_t^{\star}$ as close as possible to $\mathbf{a}_t$. From the perspective of imitation learning,  {this keeps $\mathbf{y}_t^{\star}$ faithful to the demonstrations}, thereby enabling the actions generated by DiffOG to successfully perform the demonstrated task.

2)  The second objective is to make the output satisfy the constraints in optimization problem \eqref{eq:opt2}. According to Proposition~\ref{prop:feasibility}, optimization problem \eqref{eq:opt2} is guaranteed to have a unique optimal solution, and this solution will always satisfy the constraints.

3) Due to the presence of the smooth term, $\mathbf{y}_t^{\star}$ tends to be smoother. Specifically, the sum of its derivatives across time steps tends to be smaller during optimization.\label{rem:obj}
\end{remark}

Thus, the training of DiffOG balances the three interpretable learning objectives outlined in Remark \ref{rem:obj}. In essence, this process can be summarized as learning $\mathbf{Q}$ to formulate a quadratic programming problem that produces a solution balancing smoothness, enforcing constraint satisfaction, and ensuring that $\mathbf{y}_t^{\star}$ closely aligns with $\mathbf{a}_t$.

Moreover, the reason for setting only $\mathbf{Q}$ as learnable is that the remaining variables, such as the upper and lower bounds of the constraints $d_{\text{max}}$ and $d_{\text{min}}$, and the smoothing weight $\alpha$, have more direct model-based meanings and can be manually specified based on requirements without the need for learning. Different values for the smoothing weight $\alpha$, and the bounds of the constraints $d_{\text{max}}$ and $d_{\text{min}}$ can directly and controllably influence the properties of the output $\mathbf{y}_t^{\star}$.

\subsection{Parameterize $\mathbf{Q}$ with Transformer}

\begin{figure}[t]
\centering
\begin{overpic}[trim=0 0 0 0,clip, width=0.45\textwidth]{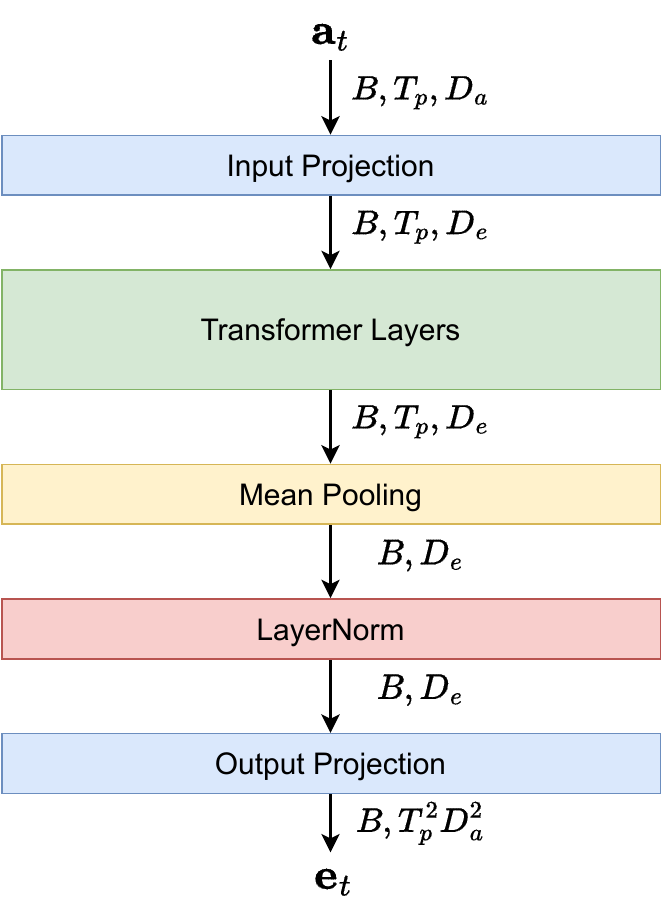}
\end{overpic}
\caption{Pipeline of the transformer encoder.}
\label{fig:model}
\end{figure}

So far, both the differentiability discussed in Remark~\ref{rem:diff} and the derivation of the interpretable learning objectives in Remark~\ref{rem:obj} require $\mathbf{Q}$ to be symmetric positive definite. This section will introduce our method for constructing a symmetric positive definite $\mathbf{Q}$.

$\mathbf{Q}$ is the learnable linear transformation. Since actions are highly diverse after sampling from  dataset $\mathcal{D}$, different action samples $\mathbf{a}_t$ may require distinct $\mathbf{Q}$ matrices to construct the trajectory optimization problem, ensuring an optimal balance of the three interpretable learning objectives described in Remark \ref{rem:obj} under the presence of smoothing terms and constraints. When the task is long-horizon, complex, and involves a high number of robot DOFs, $\mathbf{Q}$ needs to be more generalizable. 

Thanks to the well-posed properties, namely feasibility and differentiability, of our proposed trajectory optimization formulation, we are able to incorporate a transformer module to significantly enhance its capacity through a series of careful design choices. This enhancement enables the optimization process to generalize to a wider range of more complex tasks, while preserving interpretability (see Remark~\ref{rem:obj}) and maintaining feasibility and training stability (see Remark~\ref{prop:feasibility}). We use a transformer encoder $E_\theta$ to construct a symmetric positive definite $\mathbf{Q}$, as detailed in Algorithm \ref{alg:q} and Fig.~\ref{fig:firstPage}. The encoder used in Algorithm \ref{alg:q} which computes the embedding $\mathbf{e}_t = E_\theta(\mathbf{a}_t)$ is shown in Fig.~\ref{fig:model}. In Algorithm \ref{alg:q}, steps 3)–5) construct a lower triangular matrix with positive diagonal entries from the embedding output by the transformer. Additionally, the use of clamping prevents  {numerical explosion caused by exponentiation}. Based on Cholesky factorization, Step 6) results in a symmetric positive definite $\mathbf{Q}$. Step 7) further reinforces the symmetry of $\mathbf{Q}$ to mitigate the impact of minor numerical errors introduced during computation, ensuring the desired properties of $\mathbf{Q}$ are preserved. In practice, we clamp the values to the range [-10, 10]. For the small constant $\epsilon$, we use $1 \times 10^{-4}$.

By Remark \ref{rem:diff}, we know that optimization problem \eqref{eq:opt1} is differentiable, and $\nabla_{\mathbf{Q}} \mathbf{y}^{\star}_t$ can be computed. The operations in Algorithm \ref{alg:q} are also differentiable. Then $\nabla_\theta {\mathbf{Q}}$ can be computed. Therefore, based on the chain rule, DiffOG can be fully trained in an end-to-end manner. For quadratic programming solver, we utilize qpth \cite{amos2017optnet}, which enables batch-form quadratic programming computation along with efficient backpropagation.

\begin{algorithm}[t]
\caption{Constructing a symmetric positive definite matrix $\mathbf{Q}$ from $\mathbf{a}_t$}
\textbf{Input:} $\mathbf{a}_t \in \mathbb{R}^{T_pD_a}$, small constant $\epsilon > 0$.\\
\textbf{Output:} $\mathbf{Q} \in \mathbb{R}^{T_pD_a \times T_pD_a}$

\begin{enumerate}
\item $\mathbf{e}_t \gets \mathrm{Encoder}(\mathbf{a}_t)$
\item $\mathbf{L} \gets \mathbf{e}_t.\mathrm{reshape}(T_pD_a, T_pD_a)$
\item \textbf{for each diagonal entry} $\mathbf{L}_{ii}$:
      \[
      \mathbf{L}_{ii} \gets \exp(\mathbf{L}_{ii}) + \epsilon
      \]
\item $\mathbf{L} \gets \mathrm{clamp}(\mathbf{L})$ \hfill~\textbackslash \textbackslash~{prevent numerical explosion}
\item $\mathbf{L} \gets \mathrm{tril}(\mathbf{L})$ \hfill \textbackslash \textbackslash~{set upper-triangular elements to zero}
\item $\mathbf{Q} \gets \mathbf{L} \mathbf{L}^\mathrm{T} + \epsilon \mathbf{I}$
\item $\mathbf{Q} \gets \frac{1}{2} (\mathbf{Q} + \mathbf{Q}^\mathrm{T})$ \hfill \textbackslash \textbackslash~{enhance symmetry}
\end{enumerate}
\label{alg:q}
\end{algorithm}

\subsection{Inference}\label{sec:inference}
The training of DiffOG is on discrete trajectory chunks. Therefore, to fully constrain the generated trajectory during inference, the forward pass of optimization problem~\eqref{eq:opt1}  {needs} to add a new constraint. 
 \begin{align*}
\mathbf{y}^{\star}_t=&\underset{\mathbf{y}_t}{\operatorname{argmin}} \frac{1}{2} \mathbf{y}_t^{\mathrm{T}} \mathbf{Q} \mathbf{y}_t+\mathbf{a}_t^{\mathrm{T}} \mathbf{y}_t +\frac{\alpha}{2} (\frac{d\hat{\mathbf{c}}_t}{dt}) ^{\mathrm{T}} \frac{d\hat{\mathbf{c}}_t}{dt},\\
\text { subject to }~
 &d_{\text{min}} \Delta t \leq \hat{c}_{t} - \hat{c}_{t-1} \leq d_{\text{max}} \Delta t, \\
 &d_{\text{min}} \Delta t \leq \hat{c}_{t+k+1} - \hat{c}_{t+k} \leq d_{\text{max}} \Delta t, \\
 &k = 0, 1.  \ldots, T_p-2.\notag
\end{align*}
$\hat{c}_{t-1}$ is the action at the final time step from the previous iteration.

\section{Experiments}
We validated DiffOG on 13 different tasks as shown in Fig.~\ref{fig:tasks}.
Experiments involve two base policies, Diffusion Policy and 3D Diffusion Policy (DP3), along with a comprehensive comparison of trajectory processing baselines, including constraint clipping and penalty-based optimization. Additionally, we compare DiffOG with the constrained visuomotor policy \cite{xu2024leto}. We also conduct ablation studies on DiffOG.

For the hyperparameters of Diffusion Policy and DP3, we use the default parameters provided in their respective work \cite{chi2023diffusionpolicy,ze20243d}.  DiffOG, penalty-based optimization, and constraint clipping use same constraints and same base policy, either DP3 or Diffusion Policy, depending on the type of camera input used for the task. For details on the parameters of DiffOG and trajectory processing baselines, as well as information on the datasets, please refer to Appendix.

We divide the results of our experiments into four sections for presentation.

\textbf{Benchmark with Baselines Across Tasks}: We demonstrate how DiffOG optimizes the trajectories generated by Diffusion Policy and DP3 across 13 tasks. Additionally, we compare its performance with baseline methods, including constraint clipping and penalty-based trajectory optimization.

\textbf{Ablation Study on Adjustability of Trajectory}: We show how adjusting the smoothing weight and constraint bounds allows us to control the properties of the optimized actions produced by DiffOG. This section further demonstrates the interpretability of DiffOG.

\textbf{Ablation Study on Static $\mathbf{Q}$, Matrix Learning, and Transformer}: Through experiments, we demonstrate the significance of the transformer-based design in DiffOG by comparing it with static $\mathbf{Q}$ and matrix-learning-based $\mathbf{Q}$ components. The results emphasize that the transformer architecture plays a crucial role in achieving superior performance. Specifically, it enables DiffOG to fully leverage differentiability, effectively balance trajectory optimization learning objectives, and maintain high fidelity to demonstration data trajectories.

\textbf{Comparison with a Constrained Visuomotor Policy}: We present comparative experiments between DiffOG and an constrained visuomotor policy Leto \cite{xu2024leto}, showing that DiffOG achieves superior performance.

 {\textbf{Inference Time-varying Constraint Experiment}: In this experiment, the policy is trained with a fixed constraint bound but evaluated with time-varying bounds during inference, demonstrating DiffOG’s capability to generalize zero-shot to unseen constraint settings.}

 {\textbf{Residual Policy Baseline with Smoothness Penalty}: In this experiment, we train a residual policy by learning an action correction on top of base policy, with an additional smoothness penalty in the loss. By comparing the performance of DiffOG with this baseline, we further highlight how DiffOG’s structure leads to better alignment with task requirements and stronger overall performance, making it a more principled and generalizable method.}

\begin{figure*}[t]
\centering
\begin{overpic}[trim=0 0 0 0,clip, width=0.98\textwidth]{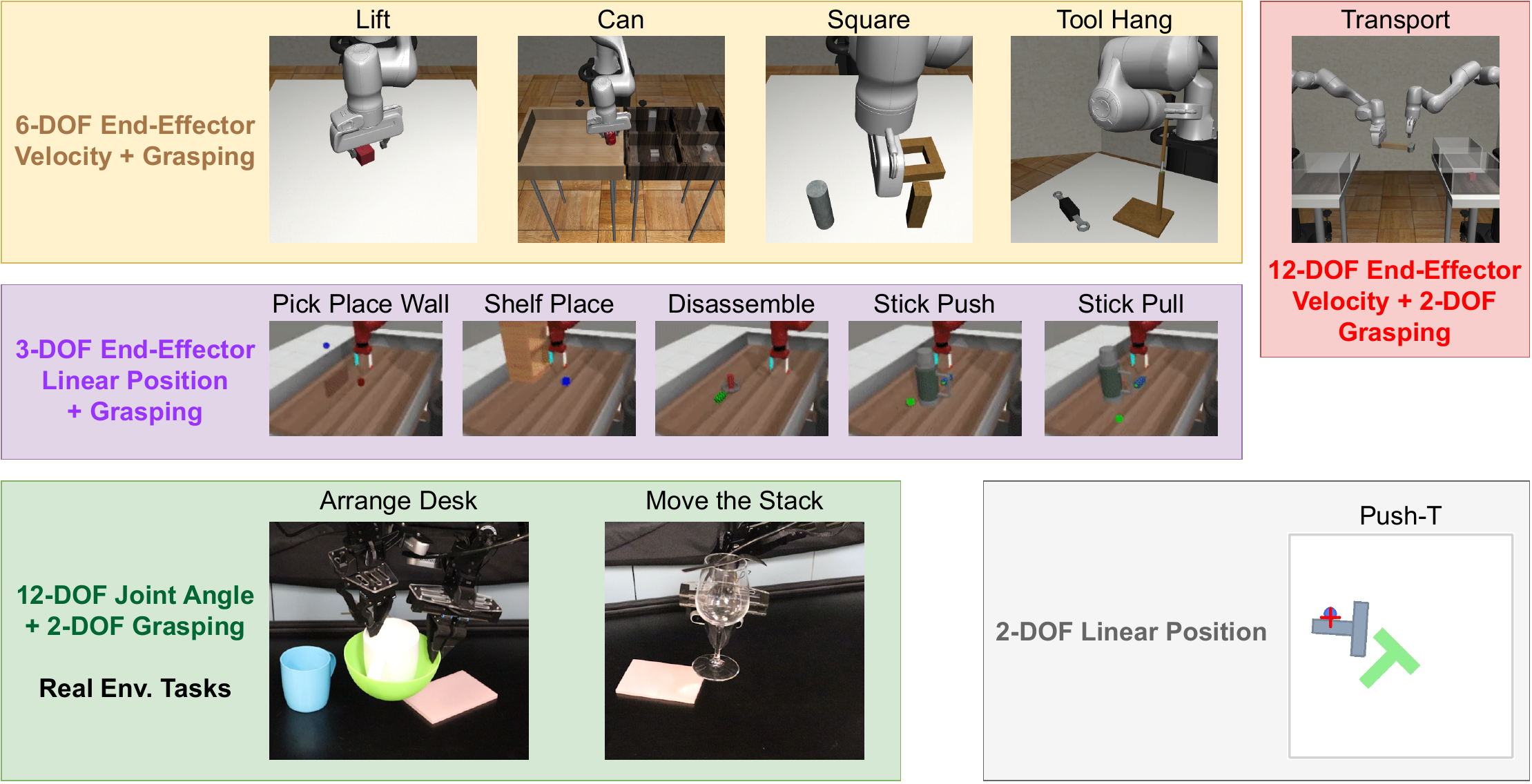}
\end{overpic}
\caption{We validated DiffOG on 13 different tasks. The 13 tasks span three types of action space.}
\label{fig:tasks}
\end{figure*}

\begin{figure*}[t]
\centering
\begin{overpic}[trim=0 0 0 0,clip, width=0.98\textwidth]{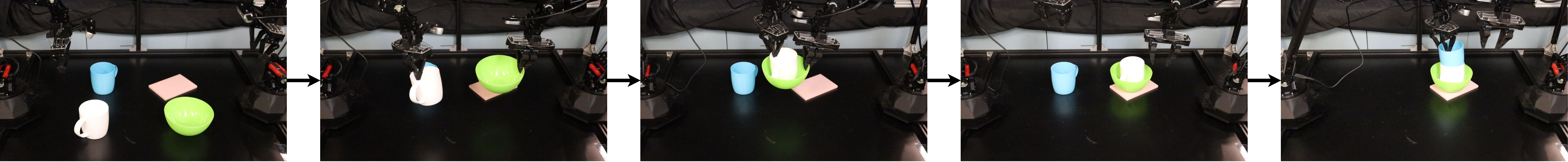}
\end{overpic}
\caption{The process of the arrange desk task.}
\label{fig:bowl_process}
\end{figure*}

\begin{figure*}[t]
\centering
\begin{overpic}[trim=0 0 0 0,clip, width=0.98\textwidth]{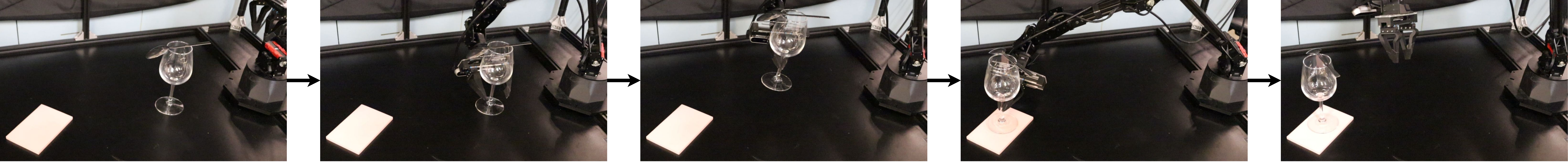}
\end{overpic}
\caption{The process of the move the stack task.}
\label{fig:glass_process}
\end{figure*}

\begin{figure*}[t]
\centering
\includegraphics[trim=0 0 0 0,clip,width=0.98\textwidth]{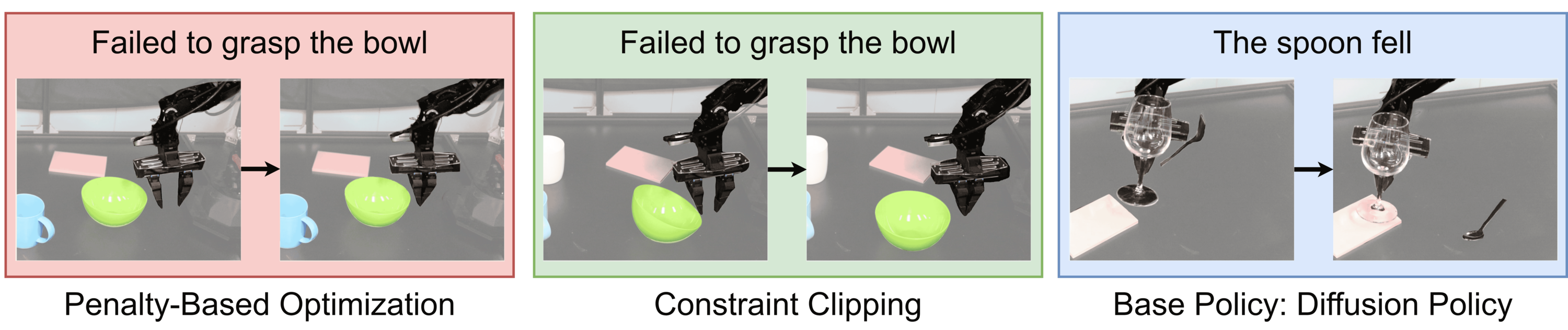}
\caption{
 {Typical failure cases of baselines.}}
\label{fig:failure}
\end{figure*}

\begin{table*}[]
\centering
\begin{tabular}{|c|c|c|c|c|c|c|}
\hline
                              & Lift             & Can              & Square           & Tool Hang        & Transport        & Push-T           \\ \hline
Base Policy: Diffusion Policy & $ {1.00}\pm {0.00} $ & {$ {0.98}\pm {0.01} $} & ${0.91}\pm {0.01}$ & $ 0.83\pm 0.03 $ & $ 0.91\pm 0.01 $ & $ 0.84\pm 0.03 $ \\ \hline
\rowcolor{gray!30}DiffOG Dataset (Ours)         & $ \textbf{1.00}\pm \textbf{0.00}$ & {$ \textbf{0.98}\pm \textbf{0.01}$}  & $ 0.87\pm 0.02 $ & $ \textbf{0.82}\pm \textbf{0.02}$  & $ 0.89\pm 0.02$  & $ \textbf{0.83}\pm \textbf{0.02} $ \\
\rowcolor{gray!30}DiffOG Refine (Ours)          & $ \textbf{1.00}\pm \textbf{0.00} $ & {$ \textbf{0.98}\pm \textbf{0.01}$} & {$ \textbf{0.90}\pm\textbf{0.01}$} &  $0.81\pm 0.02 $  &     $ \textbf{0.91}\pm\textbf{0.03} $             & $ 0.80\pm 0.02 $  \\
Constraint Clipping           & $ \textbf{1.00}\pm \textbf{0.00} $ & $ 0.93\pm 0.04 $ & $ 0.78\pm 0.03 $ & $ 0.59\pm 0.03 $ & $ 0.82\pm 0.02 $ & $ 0.82\pm 0.02 $ \\
Penalty-Based Optimization   & $ \textbf{1.00}\pm \textbf{0.00} $ & $ 0.96\pm 0.02 $ & $ 0.76\pm 0.03 $ & $ 0.79\pm 0.01 $ & $ 0.81\pm 0.02 $ & $ 0.81\pm 0.04 $ \\ \hline
\end{tabular}
\caption{Success rate of simulated tasks using Diffusion Policy \cite{chi2023diffusionpolicy} as the base policy. We report the mean success rate, calculated as the average over the final 10 checkpoints, with each checkpoint evaluated across 50 different environment initializations. The results are further averaged over 3 training seeds, totaling 150 environment evaluations. We use proficient-human demonstrations of Robomimic benchmark \cite{mandlekar2022matters}. We highlight in \textbf{bold} the method that achieves the highest success rate after applying trajectory processing.  DiffOG, penalty-based optimization, and constraint clipping use same constraints and same base policy. On challenging tasks such including square, tool hang, and transport, DiffOG outperforms penalty-based optimization with an average success rate improvement of 9\% and surpasses constraint clipping with an average success rate improvement of 14.7\%.}
\label{tab:robomimic}
\end{table*}

\begin{table*}[]
\centering
\begin{tabular}{|c|c|c|c|c|c|}
\hline
                            & Pick Place Wall  & Shelf Place      & Disassemble      & Stick Push       & Stick Pull       \\ \hline
Base Policy: DP3            & $ 0.98\pm 0.01 $ & $ 0.77\pm 0.04 $ & $ 0.87\pm 0.02$  & $ 1.00\pm 0.00 $ & $ 0.70\pm 0.02 $  \\ \hline
\rowcolor{gray!30} DiffOG Dataset (Ours)       & $ \textbf{0.98}\pm \textbf{0.01} $ & $ \textbf{0.73}\pm \textbf{0.01} $ & $ 0.89\pm 0.05$  & $ \textbf{1.00}\pm \textbf{0.00} $ & $ \textbf{0.70}\pm \textbf{0.03} $ \\
\rowcolor{gray!30} DiffOG Refine (Ours)        & $ \textbf{0.98}\pm \textbf{0.01} $ & $ 0.72\pm 0.02 $ & {$ \textbf{0.90}\pm \textbf{0.03}$}  & $ \textbf{1.00}\pm \textbf{0.00} $ & $ 0.68\pm 0.02 $  \\
Constraint Clipping         & $ \textbf{0.98}\pm \textbf{0.01} $ & $ 0.70\pm 0.01 $ & $ 0.86\pm 0.05 $ & $ \textbf{1.00}\pm \textbf{0.00} $ & $ 0.69\pm 0.02 $  \\
Penalty-Based Optimization & $ \textbf{0.98}\pm \textbf{0.01} $ & $ 0.67\pm 0.03 $ & $ 0.86\pm 0.06 $ & $ \textbf{1.00}\pm \textbf{0.00} $ & $ \textbf{0.70}\pm \textbf{0.03} $ \\ \hline
\end{tabular}
\caption{Success rate of simulated tasks using DP3 \cite{ze20243d} as the base policy. We report the mean success rate, calculated as the average over the final 10 checkpoints, with each checkpoint evaluated across 50 different environment initializations. The results are further averaged over 3 training seeds, totaling 150 environment evaluations. We highlight in \textbf{bold} the method that achieves the highest success rate after applying trajectory processing. All tasks are trained using 30 demonstrations.  DiffOG, penalty-based optimization, and constraint clipping use same constraints and same base policy.}
\label{tab:dp3}
\end{table*}

\begin{table*}[t]
\centering
\begin{tabular}{|c|c|>{\columncolor{gray!30}}c|c|c|}
\hline
               & Base Policy: Diffusion Policy & DiffOG Dataset & Constraint Clipping & Penalty-Based Optimizaiton \\ \hline
Arrange Desk   & \textbf{9/15}                          & \textbf{9/15}            & 2/15                & 4/15                        \\
Move the Stack & 3/15                          & \textbf{10/15}           & 6/15                & 7/15                        \\ \hline
\end{tabular}
\caption{Success rate of real tasks using Diffusion Policy \cite{ze20243d} as the base policy and using ALOHA \cite{zhao2023learning} as experimental platform. We highlight in \textbf{bold} the method that achieves the highest success rate.}
\label{tab:real}
\end{table*}

\begin{table}[t]
\centering
\begin{tabular}{|c|c|c|}
\hline
                              & Arrange Desk& Move the Stack\\ \hline
Base Policy: Diffusion Policy & 0.153 s                      & 0.070 s                       \\ \hline
DiffOG  Layer                   & 0.064 s                      & 0.047 s                       \\ \hline
\end{tabular}
\caption{ We report the average inference time over 50 inferences for the base policy and DiffOG on 3080. For diffusion policy, we use 16 DDIM \cite{song2021denoising} steps for arrange desk and 8 DDIM steps for move the stack, with arrange desk requiring more due to its longer horizon and higher DoF. Results demonstrate DiffOG's real-time feasibility.}
\label{tab:real_time}
\end{table}

\begin{figure*}[t]
    \centering
    \subfloat[ Max acceleration on Robomimic.]{%
        \includegraphics[width=0.45\textwidth]{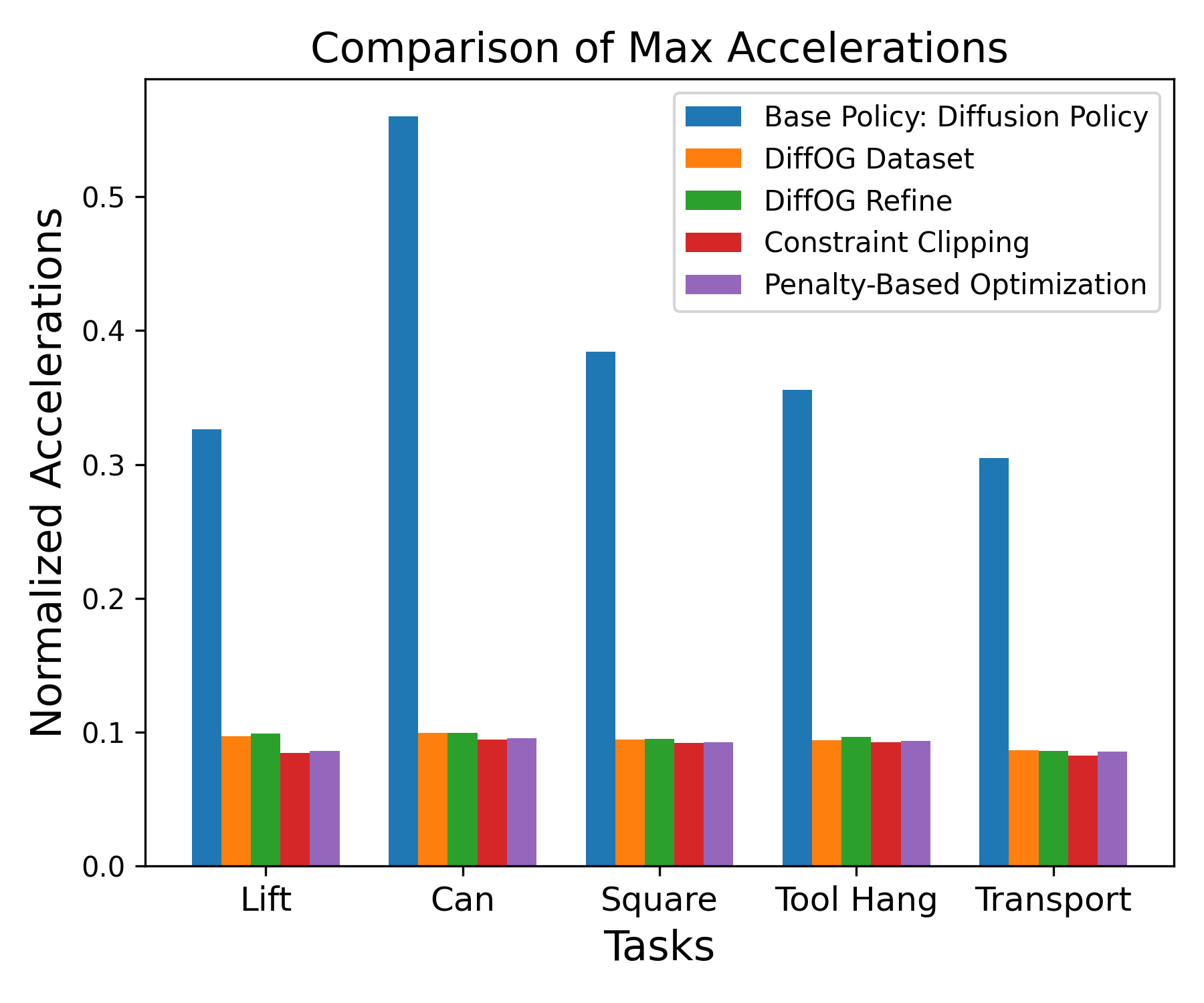}%
        }
    \subfloat[The standard deviation of acceleration on Robomimic.]{%
        \includegraphics[width=0.45\textwidth]{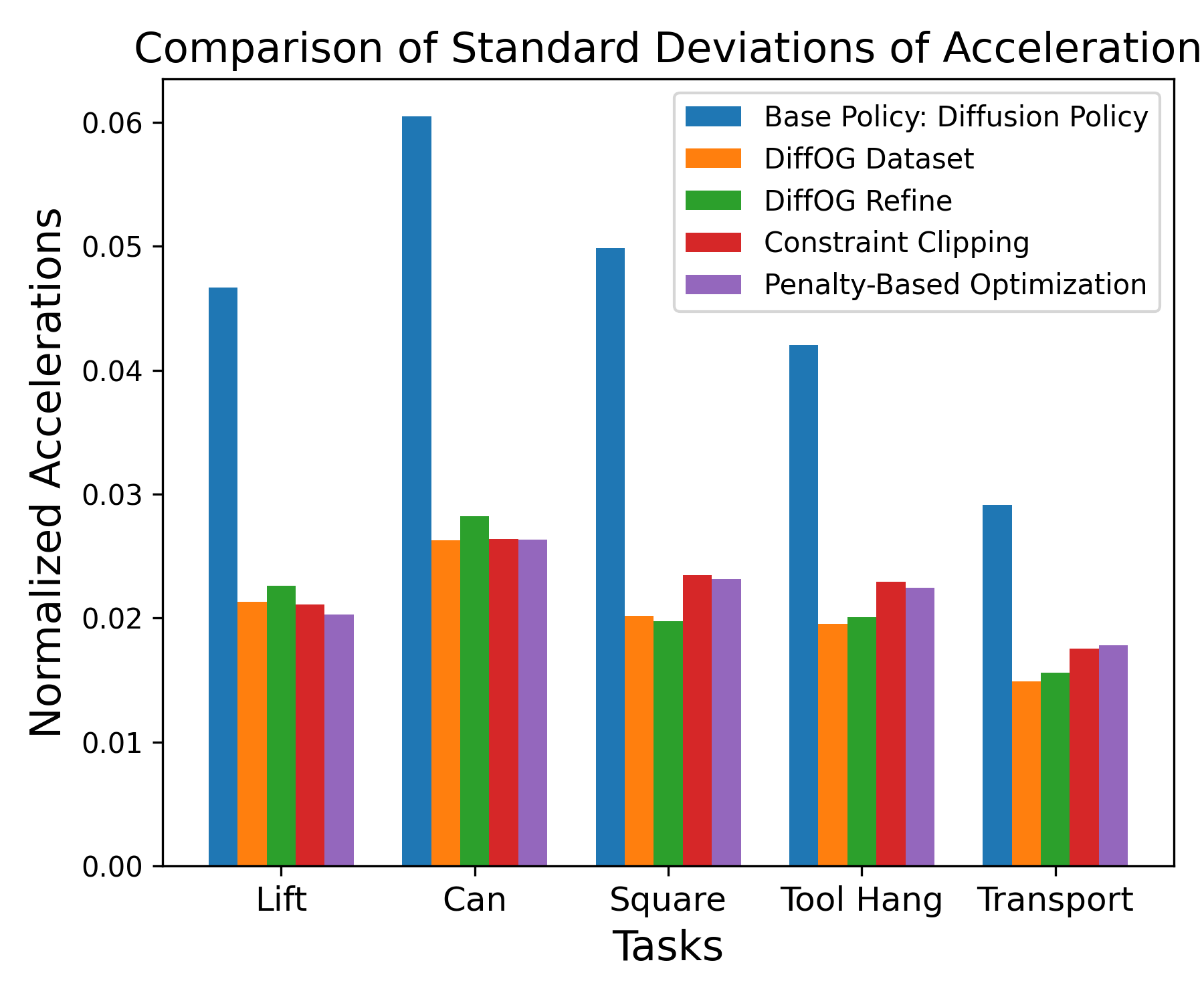}%
        }\\
    \subfloat[ Max velocity on Meta-World.]{%
        \includegraphics[width=0.5\textwidth]{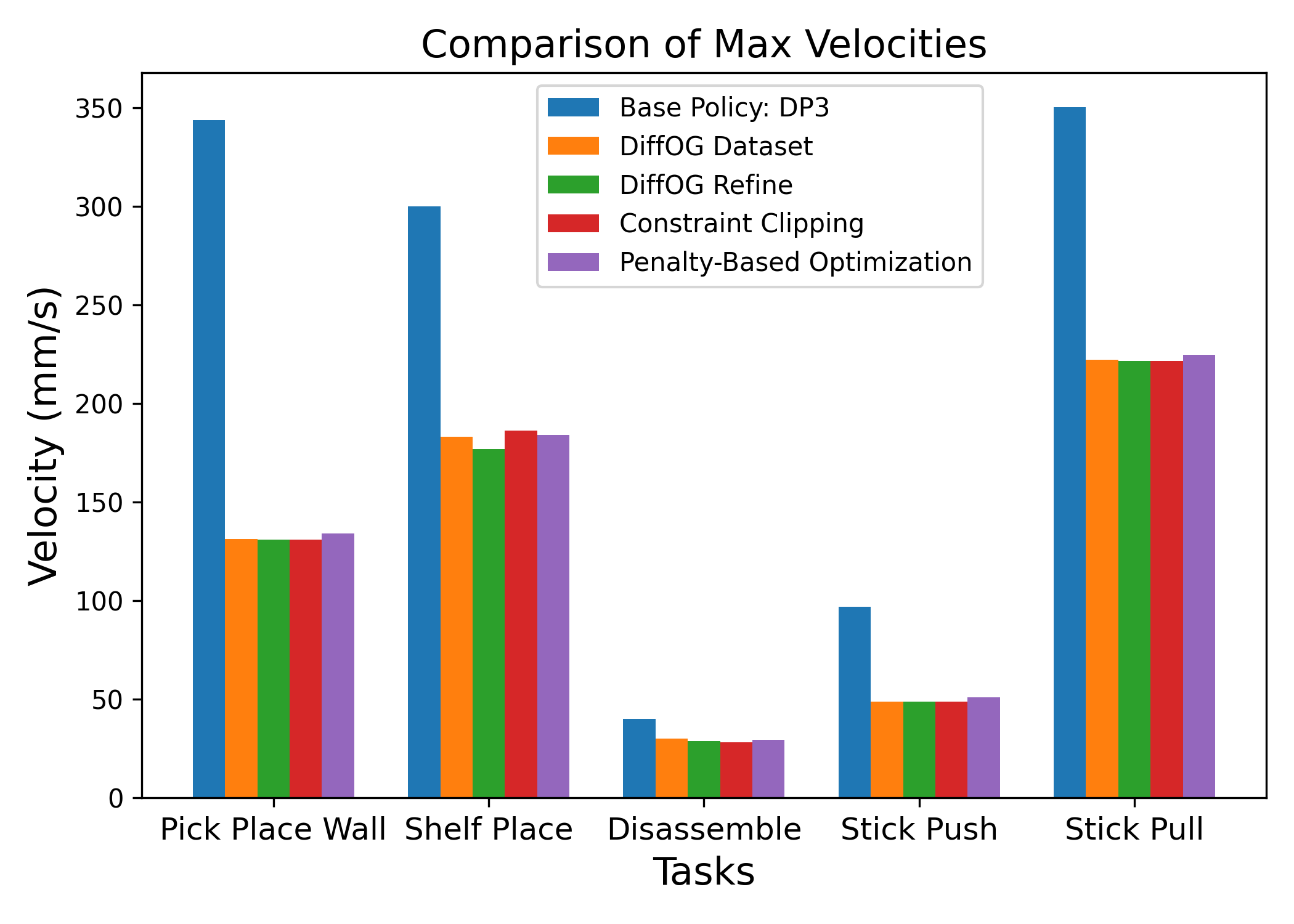}%
        }
    \subfloat[The standard deviation of velocity on Meta-World.]{%
        \includegraphics[width=0.5\textwidth]{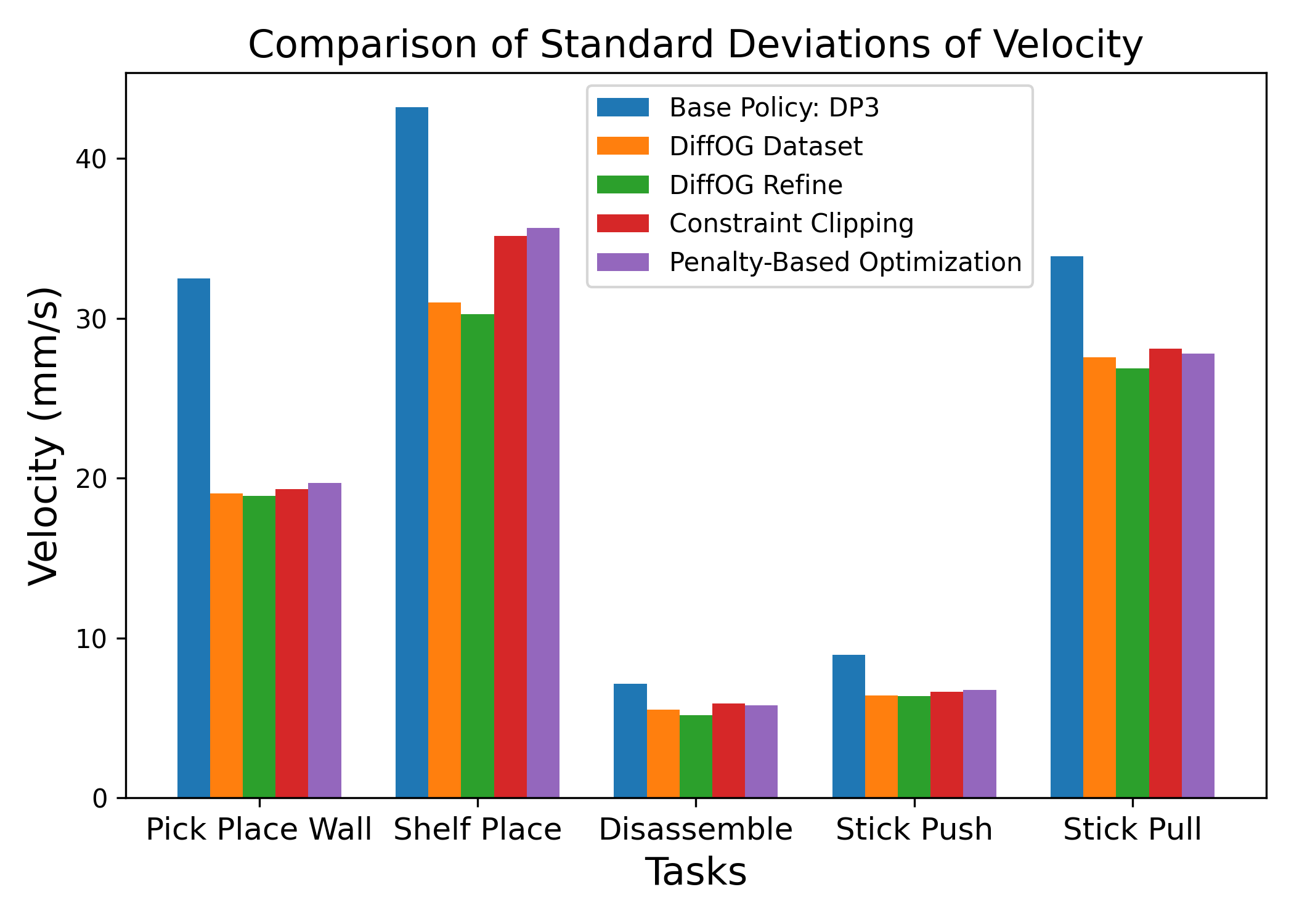}%
        }\\
    \caption{Comparison of trajectory metrics across methods on Robomimic and Meta-World benchmarks. These results are derived from 1500 rollouts (across 50 environments, 3 seeds, and the last 10 checkpoints per seed), providing strong statistical significance.  {Note that these values do not represent simple variations around a common mean and therefore cannot be interpreted in the same way as conventional mean $\pm$ standard deviation error bars, as they are not computed from a single rollout trajectory. We report them separately to highlight distinct aspects of trajectory quality.} For Robomimic tasks, the larger maximum acceleration and greater standard deviation of acceleration over time indicate less smooth trajectories. Similarly, for Meta-World tasks, this observation applies to the velocity metrics. Since the actions provided by the Robomimic dataset are already normalized \cite{mandlekar2022matters}, the acceleration metrics shown here for Robomimic are also normalized. For details on the calculation of acceleration and velocity metrics, as well as the details to unnormalize the Robomimic metrics, please refer to Appendix.}
    \label{fig:traj}
\end{figure*}

\begin{figure*}[t]
    \centering
    \subfloat[ Comparison of trajectory metrics on push-T.]{%
        \includegraphics[width=0.35\textwidth]{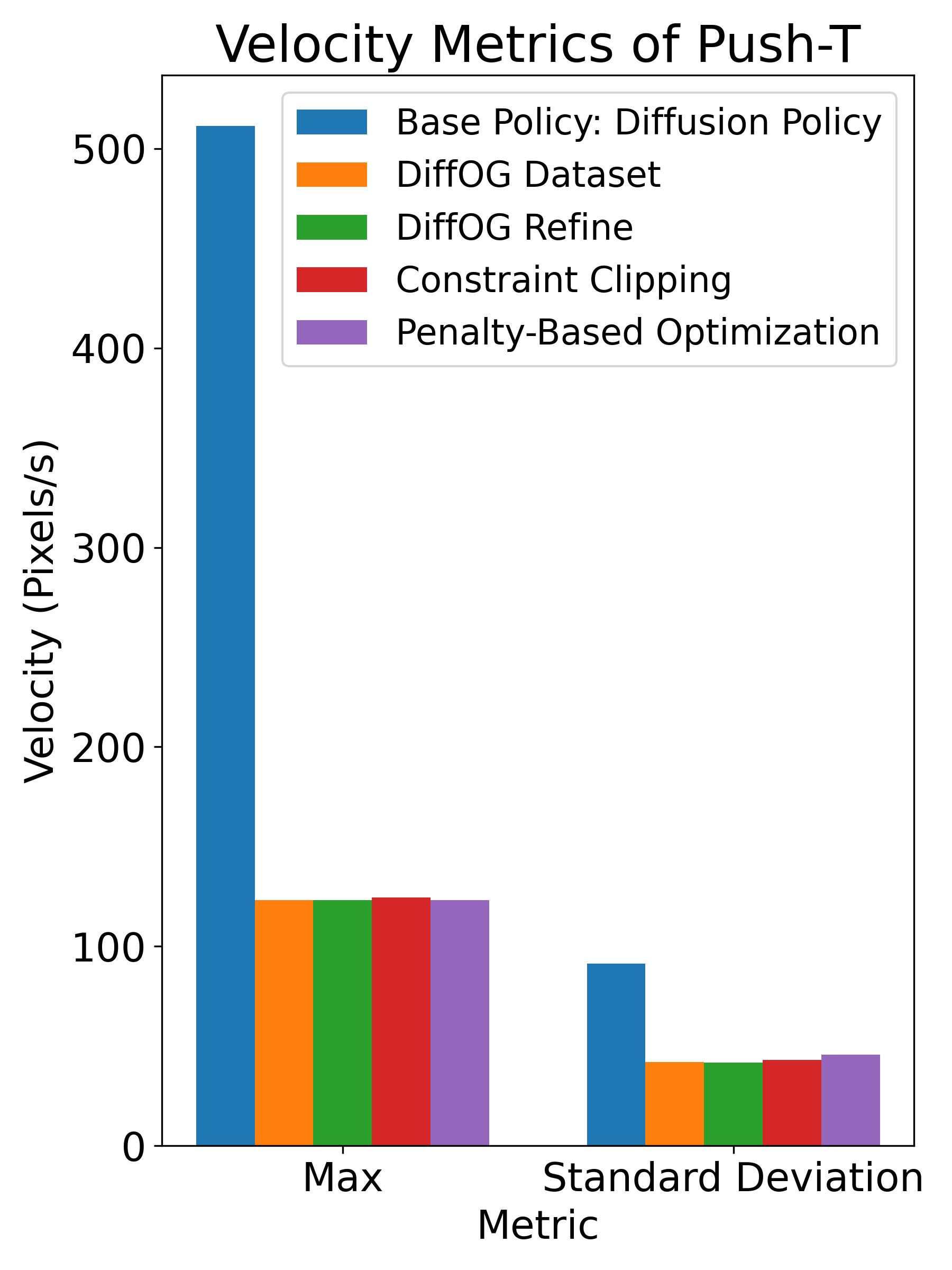}%
        }
    \subfloat[Comparison of trajectory metrics on real-world tasks.]{%
        \includegraphics[width=0.65\textwidth]{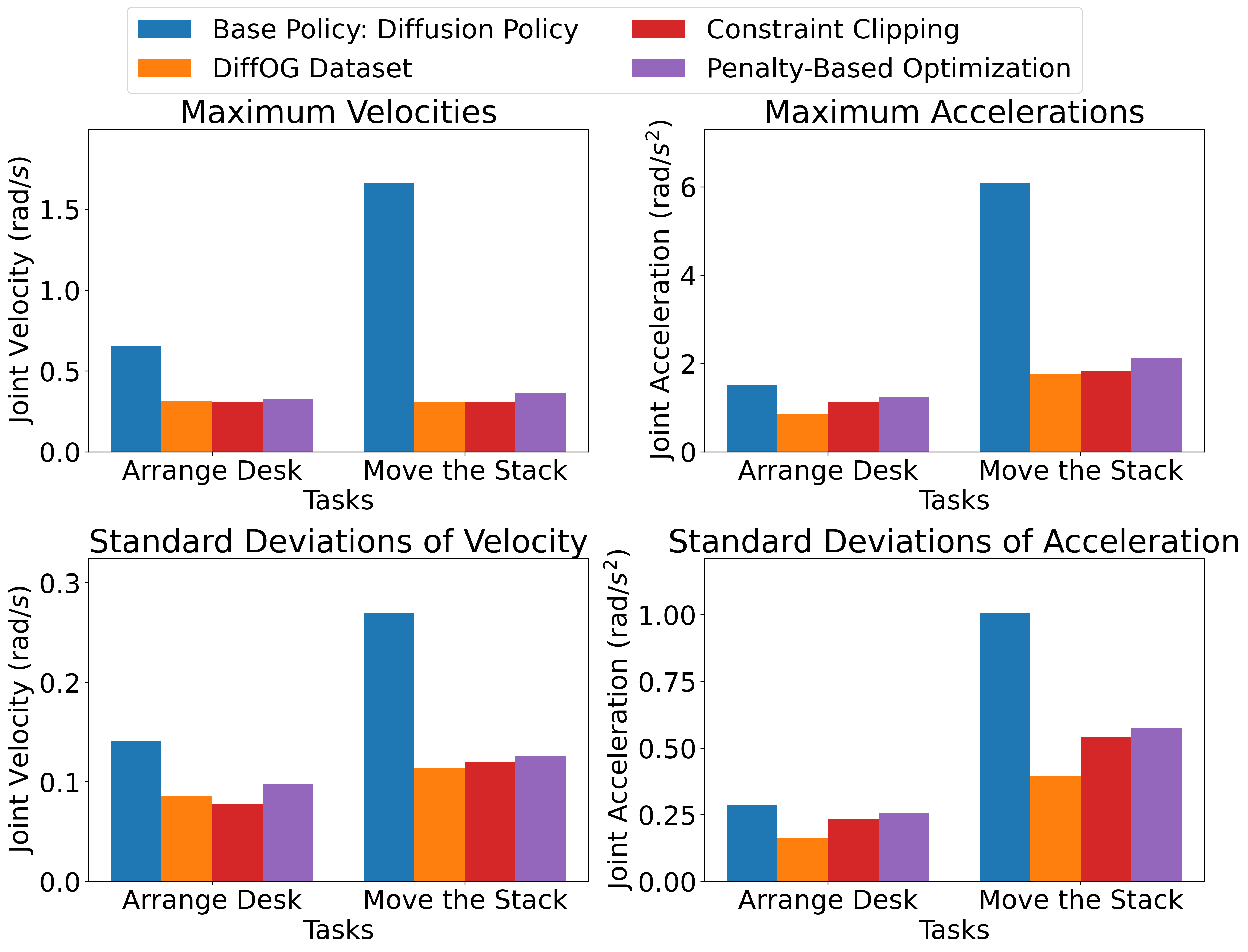}%
        }
    \caption{Comparison of trajectory metrics across methods on push-T and real-world tasks. The larger maximum velocity and acceleration and greater standard deviation of velocity  and acceleration over time indicate less smooth trajectories.  For details on the calculation of acceleration and velocity metrics, please refer to  Appendix.}
        \label{fig:traj2}
\end{figure*}

\begin{figure*}[t]
    \centering
    \subfloat[The smoothing weight varies while the constraint bound remains fixed at 0.1.]{%
        \includegraphics[width=0.98\textwidth]{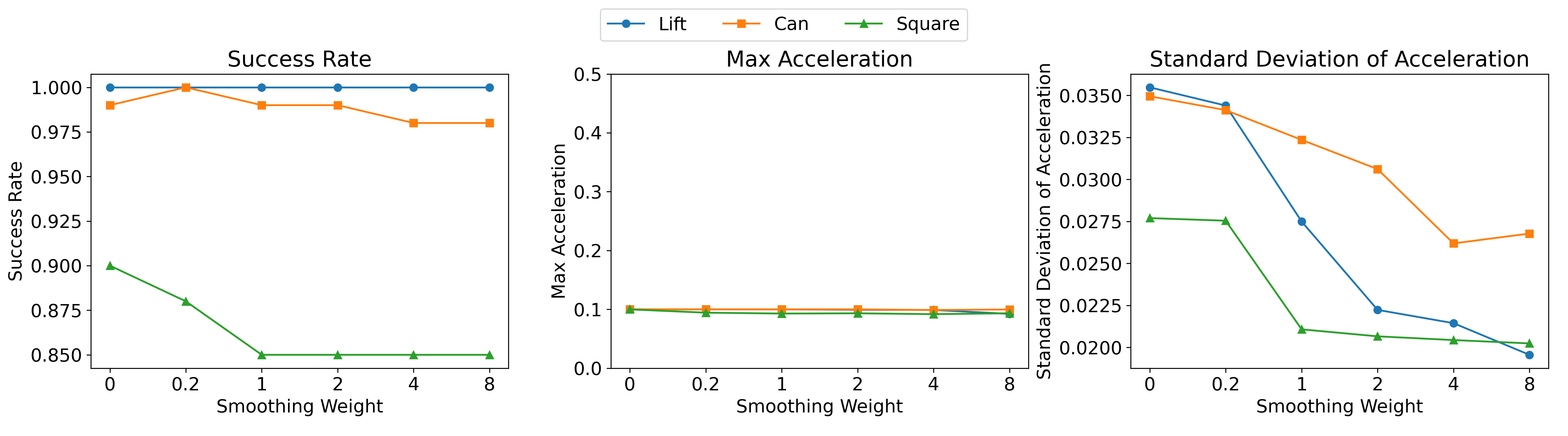}%
        }\\
    \subfloat[ The constraint bounds vary while the smoothing weight $\alpha$ remains fixed at 4.]{%
        \includegraphics[width=0.98\textwidth]{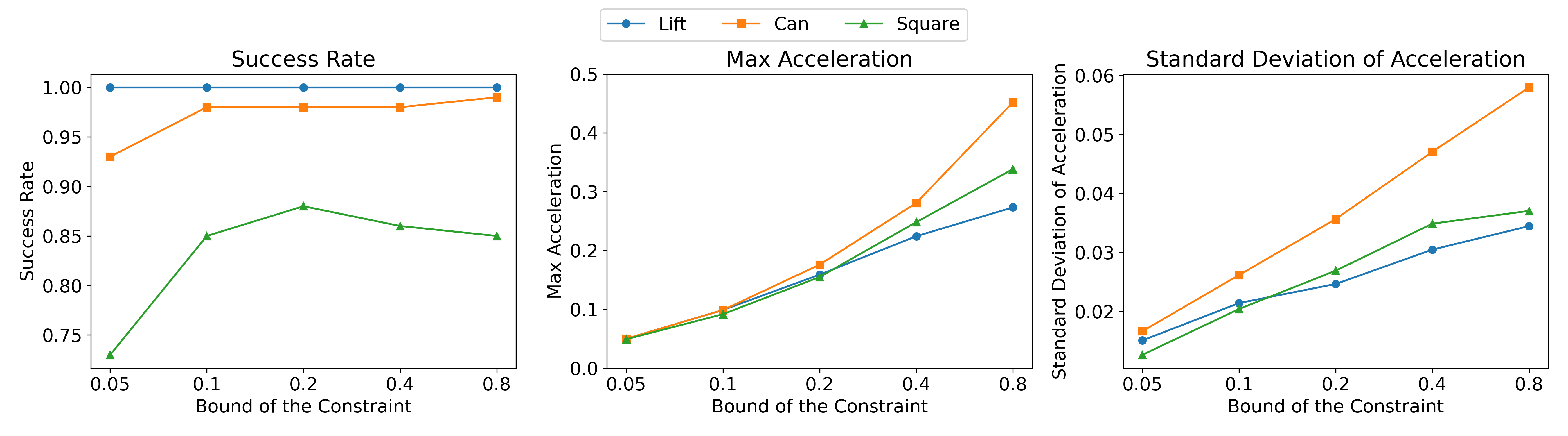}%
        }
    \caption{Ablation study on adjustability of trajectory. We use DiffOG dataset training pipeline as shown in Fig.~\ref{fig:overview}. For success rate, we report the mean success rate calculated as the average over the final 10 checkpoints, with each checkpoint evaluated across 50 different environment initializations. Since our primary goal is to observe the trend in metric variations as the smoothing weight and constraint bound change, we conducted experiments using a single seed, leading to slight variations in the mean success rate compared to the three-seed results reported in Table~\ref{tab:robomimic}. }
    \label{fig:control}
\end{figure*}

\subsection{Benchmark with Baselines Across Tasks}\label{sec:benchmark}
In this section, we demonstrate how DiffOG optimizes the trajectories generated by base policies across 13 tasks. The 13 tasks span three types of action space.  {See \hyperref[sec:task]{Appendix C} for more details.}

The purpose of the experiments  {is} to investigate the impact of applying trajectory processing, which enforces specific constraints on action trajectories, on policy performance across a wide range of tasks. Additionally, we aim to evaluate whether this approach effectively makes action trajectories smoother and more constrained. The experimental results are  shown in Tables~\ref{tab:robomimic}, \ref{tab:dp3}, \ref{tab:real} and Figs.~\ref{fig:traj} and \ref{fig:traj2}. 

 {We compute evaluation metrics based on 1500 rollouts (last 10 checkpoints $\times$ 3 seeds $\times$ 50 environment conditions) to ensure statistical significance. This is important for evaluating both task success rate, which indicates whether the method still effectively captures the distribution of demonstration data under constraints, and motion-related characteristics, which reflects trajectory quality.} By analyzing the experimental results, we can draw the following conclusions.

1. In the simulation benchmarks, we evaluate trajectory smoothness by analyzing the first-order derivatives of actions. Specifically, we compare the maximum values and standard deviations of these derivatives throughout the policy rollouts. In Robomimic, the first-order derivative corresponds to the second derivative of axis-angle, while in Meta-World and push-T, it represents  {linear velocity}. For the real-world tasks, in addition to computing the maximum and standard deviation of the first-order derivatives (i.e., joint angular velocities), we also assess the second-order derivatives (joint angular accelerations) using the same metrics. 

These values serve as indicators of trajectory smoothness: the jerkier a trajectory is, the more abrupt and random the variations in the derivatives of the quantities represented in the action space, which is captured by higher standard deviations. Furthermore, jerky trajectories tend to exhibit sharp peaks in velocity or acceleration when sudden changes occur. Therefore, comparing the maximum values of velocity and acceleration also reflects the degree of smoothness.

Trajectory metrics are evaluated under a unified standard, with rollouts terminating upon task completion. To ensure fair and statistically meaningful comparisons, we perform extensive rollouts for each method to obtain quantitative results. For trajectory metrics of simulation benchmarks, we report the average over 1500 rollouts per method (across 50 environments, 3 seeds, and the last 10 checkpoints per seed). 

According to Figs.~\ref{fig:traj} and \ref{fig:traj2}, both DiffOG, constraint clipping, and penalty-based trajectory optimization effectively constrain action trajectories. For Meta-World tasks, DiffOG and trajectory processing baselines achieve lower maximum velocity and reduced velocity standard deviation compared to the base policy trajectories. Similarly, for velocity control tasks from the Robomimic benchmark, both the maximum acceleration and the standard deviation of acceleration are lower when using DiffOG and trajectory processing baselines compared to the base policy trajectories.  Moreover, results of real-world tasks consistently indicate that applying DiffOG and trajectory processing baselines reduces these metrics. 

The comparison of these metrics demonstrates that DiffOG, constraint clipping, and penalty-based trajectory optimization can all optimize the action trajectories generated by the policy, making them smoother and more constrained.

2. 
Based on Tables~\ref{tab:robomimic}, \ref{tab:dp3}, and \ref{tab:real}, we conclude that compared to constraint clipping and penalty-based trajectory optimization, DiffOG generates action trajectories that not only satisfy the constraints but also maintain fidelity to the dataset. As a result, DiffOG outperforms constraint clipping and penalty-based trajectory optimization across all tasks. This demonstrates that DiffOG effectively minimizes the impact on the base policy while further optimizing action trajectories to make them smoother and more constraint-compliant.  

3. Although constraint clipping and penalty-based optimization achieve acceptable performance on some tasks, they lead to a significant decline in base policy performance for more challenging tasks, such as square, tool hang, transport, and arrange desk. This decline occurs because the base policy was not trained with constraints as part of its learning objectives. Applying constraint clipping or penalty-based trajectory optimization introduces information loss, resulting in action trajectories that deviate from the demonstration dataset.

As shown in Fig.~\ref{fig:failure}, a common failure in the arrange desk task is the inability to grasp the bowl. This issue arises because trajectory processing baselines cause the actions to deviate from the fidelity of the demonstration, leading to task failures.

In contrast, DiffOG, through a data-driven approach, effectively balances the objectives of trajectory optimization and fidelity to the dataset. Consequently, for these challenging tasks, DiffOG significantly outperforms both constraint clipping and penalty-based trajectory optimization. On the four challenging tasks: square, tool hang, transport, and arrange desk, DiffOG achieves an average success rate that is 22.7\% higher than constraint clipping and 15.1\% higher than penalty-based optimization. Compared to the baselines, DiffOG consistently delivers good performance across both simple and challenging tasks.

4. 
For tasks where trajectory constraints and smoothness are crucial, such as move the stack in Table~\ref{tab:real}, the base policy often fails due to the jerkiness of the action trajectories, resulting in a high failure rate.  {The jerky behavior of base policy mainly arises from two factors: 1) lack of explicit trajectory constraints, and 2) absence of smoothness priors in the policy output. These are common issues in end-to-end neural policies, especially for manipulation tasks requiring smooth, precise, and constraint-aware motions.} In contrast, DiffOG significantly improves trajectory quality, leading to better policy performance. 

5. We report the inference time of DiffOG in real-world experiments, as shown in Fig. \ref{tab:real_time}, demonstrating DiffOG's real-time feasibility. 

\subsection{Ablation Study on Adjustability of Trajectory}\label{sec:control}
DiffOG combines model-based and data-driven designs through differentiable trajectory optimization, offering strong interpretability. In this section, we present an ablation study to demonstrate how changes in the smoothing weight $\alpha$ and the constraint bounds $d_{\text{max}} \Delta t$ and $d_{\text{min}} \Delta t$ affect the action trajectories generated by the differentiable trajectory optimization learned by DiffOG. The results are shown in Fig.~\ref{fig:control}. 

Fig.~\ref{fig:control}(a) corresponds to the scenario where the smoothing weight varies while the constraint bound remains fixed at 0.1. From Fig.~\ref{fig:control}(a), we observe that an excessively large smoothing weight slightly reduces the policy's performance. This is reasonable because when the smoothing weight is very large, the objective during DiffOG's learning process prioritizes fitting the demonstration with extremely smooth trajectories. In such cases, the fidelity of the learned trajectories to the demonstration decreases. 

On the other hand, we observe that when the smoothing weight changes but the constraint bounds remain unchanged, the generated trajectories consistently adhere to the constraints. However, as the smoothing weight increases, the trajectories become progressively smoother, as evidenced by a decreasing standard deviation of acceleration. Thus, changes in the smoothing weight directly affect the characteristics of the generated trajectories. 

In practical applications of DiffOG, it is essential to select an appropriate smoothing weight that balances achieving superior smoothness with maintaining policy performance. Based on Fig.~\ref{fig:control}(a), we selected a smoothing weight of $\alpha=4$ for the Robomimic benchmark.

Fig.~\ref{fig:control}(b) corresponds to the scenario where the constraint bounds vary while the smoothing weight remains fixed at 4. From Fig.~\ref{fig:control}(b), we observe that overly small constraint bounds lead to a decline in policy performance. This is because when the constraint bounds are very tight, the objective during DiffOG's learning process becomes fitting the demonstration with highly constrained trajectories. Under such conditions, the fidelity of the learned trajectories to the demonstration decreases.

When we adjust the constraint bounds, the maximum acceleration of the generated trajectory changes accordingly, and the smoothness of the trajectory is also affected. This indicates that in practical applications of DiffOG, it is crucial to select appropriate constraint bounds that effectively constrain the trajectory while maintaining policy performance. Based on Fig.~\ref{fig:control}(b), we selected a constraint bound of 0.1 for the Robomimic benchmark.

 {While the smoothing weight $\alpha$ and constraint bounds are manually selected, their selection can be guided by interpretable trends revealed in Fig.~\ref{fig:control}. Rather than hindering practical deployment, their interpretability, robustness, and task-specific tunability make DiffOG a flexible and practical framework for real-world use.}

\subsection{Ablation Study on Static $\mathbf{Q}$, Matrix Learning, and Transformer}\label{sec:tf}

From equation \eqref{eq:transformation}, we can observe that if we set $\mathbf{Q} = -\mathbf{I}$, the term $\bigl\|\mathbf{R} \mathbf{y}_t - \mathbf{g}_t\bigr\|^2$ in optimization problem \eqref{eq:opt2} simplifies to $\bigl\|\mathbf{y}_t - \mathbf{a}_t\bigr\|^2$, because $\mathbf{Q} = \mathbf{R}^\mathrm{T} \mathbf{R}$ and $\mathbf{g}_t = -  (\mathbf{R}^\mathrm{T})^{-1}  \mathbf{a}_t$. Under this setting, \eqref{eq:opt2} becomes a trajectory optimization problem with both a smoothing cost and a distance cost. On the other hand, instead of using a transformer as in DiffOG, we can also directly learn the matrix $\mathbf{Q}$. In this section,  we conduct experiments on the following three different cases: 

1. When $\mathbf{Q} = -\mathbf{I}$ and is kept static. In this case, the trajectory optimization is non-learnable.

2. When $\mathbf{Q} = \mathbf{L} \mathbf{L}^\mathrm{T} + \epsilon \mathbf{I}$ and the matrix $\mathbf{L}$ is directly learned. In this case, the trajectory optimization is differentiable, but its composition consists solely of matrices without neural networks. Both dataset training and refinement training can be performed in this setup.

3. DiffOG, where $\mathbf{Q}$ is generated by a transformer. In this case, the trajectory optimization is differentiable. Both dataset training and refinement training can also be performed.

We selected the  {tool hang and transport as the test tasks. These are highly challenging long-horizon manipulation tasks, which make trajectory optimization difficult due to the need to generalize across diverse trajectory patterns over extended time horizons while maintaining precise control.} The experimental results are shown in Table~\ref{tab:matrix_comp}. From the table, we observe that DiffOG achieves the best performance  {for both tasks}. 

The explanation behind this is that for static $\mathbf{Q}$ and matrix-learning-based $\mathbf{Q}$, trajectory optimization cannot effectively adjust to different trajectory inputs. This highlights the flexibility and generalization capability of DiffOG  {through transformer-based embeddings.}

Additionally, during refine training, due to the presence of a multi-modal action distribution in demonstration dataset \cite{chi2023diffusionpolicy, florence2022implicit}, the actions generated by the pre-trained policy may differ from the targets in the dataset. This conflict between the targets and inputs does not exist in dataset training paradigm of DiffOG and matrix learning. From the results of matrix-learning refine and DiffOG refine, we can observe that trajectory optimization based on matrix learning struggles to handle such conflict. However, DiffOG can handle this conflict with transformer encoder.

 {Moreover, as shown in Table~\ref{tab:inf_abl}, the computational costs of solving a static $\mathbf{Q}$, optimizing with a matrix-based representation (matrix learning), and solving the transformer-based optimization (DiffOG) are comparable. This suggests that the additional inference time compared to base policy primarily stems from solving the optimization problem, rather than from the transformer forward pass.}

\begin{table*}[t]
\centering
\begin{tabular}{|c|>{\columncolor{gray!30}}c|>{\columncolor{gray!30}}c|c|c|c|}
\hline
          & DiffOG Dataset & DiffOG Refine & Static $\mathbf{Q}$ & Matrix Learning Dataset & Matrix Learning Refine \\ \hline
 {Tool Hang} &       {\textbf{0.82}}        &     {0.81}       &      {0.73}       &        {0.77}      &  {0.60}\\ \hline
Transport &      0.89        &    \textbf{0.91}       &     0.84       &       0.85      & 0.14\\ \hline
\end{tabular}
\caption{  {Success rate of different $\mathbf{Q}$ designs on tool hang and transport task.}}
\label{tab:matrix_comp}
\end{table*}

\begin{table*}[t]
\centering
 {
\begin{tabular}{|c|c|c|c|c|}
\hline
                     & Base Policy: Diffusion Policy & DiffOG + Base & Static $\mathbf{Q}$ + Base & Matrix Learning + Base \\ \hline
Batch Inference Time & 0.13 s      & 0.28 s & 0.28 s              & 0.27 s          \\ \hline
\end{tabular}
}
\caption{ {We report the batched average inference time across three simulated tasks: lift, can, and square. Each measurement is based on a batch size of 25 and averaged over 50 inference runs. All methods are evaluated using 100 DDPM steps \cite{ho2020denoising}.} }
\label{tab:inf_abl}
\end{table*}

\begin{figure*}[h]
\centering
\begin{overpic}[trim=0 0 0 0,clip, width=0.98\textwidth]{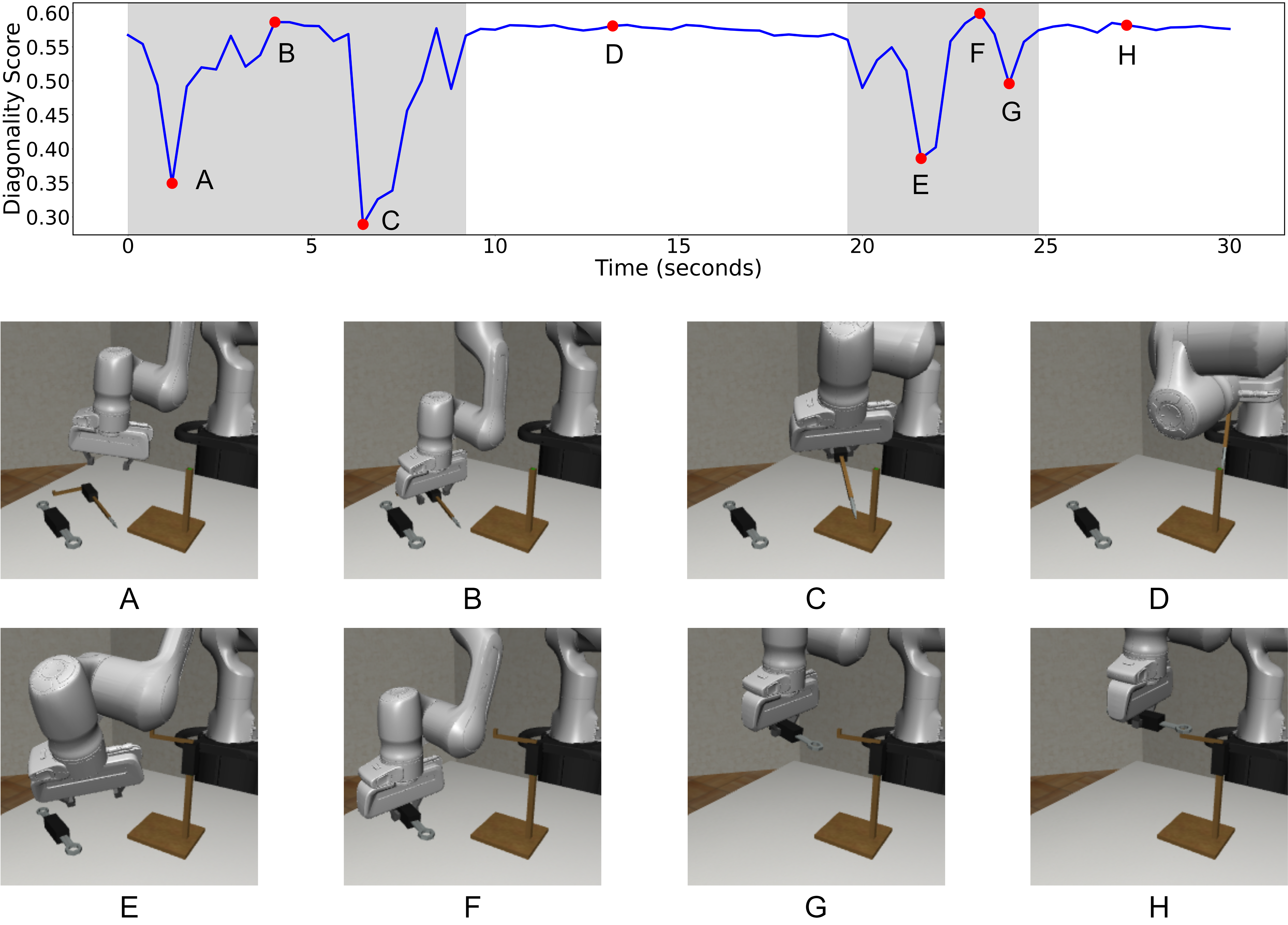}
\end{overpic}
\caption{ {Diagonality score of the learned cost matrix $\mathbf{Q}$ throughout the tool hang process. $\mathbf{Q}$ becomes more diagonal during fine-manipulation phases (B: grasping, D: insertion, F: grasping, H: hanging), and less diagonal during fast transitions, revealing how DiffOG adapts optimization structure to task dynamics.} }
\label{fig:diag}
\end{figure*}

 {Lastly, to demonstrate what the learned transformer and $\mathbf{Q}$ captures, we visualize its structure on the tool hang task. For each forward step, we compute a diagonality score defined as $\sum_{i} Q_{ii}^{2}/(\sum_{i,j} Q_{ij}^{2} + \epsilon)$, where $Q_{ij}$ denotes the $(i,j)$-th entry of the matrix $\mathbf{Q}$, and $\epsilon$ is a small constant (e.g., $1 \times 10^{-4}$ ) added for numerical stability. A larger diagonality score indicates that $\mathbf{Q}$ is more diagonal, while a smaller value suggests stronger off-diagonal coupling.
 Fig.~\ref{fig:diag} shows that $\mathbf{Q}$ becomes more diagonal during fine-manipulation phases (B: grasping, D: insertion, F: grasping, H: hanging), and less diagonal during rapid transition phases. Intuitively, when the policy already outputs actions of low magnitude, a diagonal $\mathbf{Q}$ preserves these fine-grained action characteristics. In contrast, during fast movement phases, off diagonal elements allow the optimizer to capture correlations across different action dimensions and time steps, enabling smoother global trajectories. This example illustrates how DiffOG leverages data-driven structure in $\mathbf{Q}$ to retain task relevant policy behavior when appropriate and to inject coupling when necessary, thereby enhancing generalizability.}

\subsection{Comparison with a Constrained Visuomotor Policy}

We conducted comparison experiments between DiffOG and constrained visuomotor policy Leto \cite{xu2024leto} on 7 tasks. The results are shown in Table~\ref{tab:leto_comp}. The results in the table compare DiffOG and Leto under the same constraint bounds. 

Based on the results in Table~\ref{tab:leto_comp}, we can see that the performance of policies enhanced by DiffOG is superior to Leto policies. This demonstrates the generalization capability of DiffOG across diverse tasks.

\begin{table}[t]
\centering
\begin{tabular}{|c|>{\columncolor{gray!30}}c|>{\columncolor{gray!30}}c|c|}
\hline
          & DiffOG Dataset & DiffOG Refine & Leto \\ \hline
Square    &    0.87         &     \textbf{0.90}          &   0.75   \\
Tool Hang &      \textbf{0.82}          &    0.81           &   0.28   \\
Push-T    &       \textbf{0.83}         &       0.80        &     0.64 \\ 
Pick Place Wall    &       \textbf{0.98}         &       \textbf{0.98}       &     0.63 \\
Shelf Place   &       \textbf{0.73}         &       0.72        &     0.15 \\
Disassemble    &       {0.89}         &       \textbf{0.90}         &     0.55 \\
Stick Pull     &       \textbf{0.70}         &       0.68        &     0.20 \\\hline
\end{tabular}
\caption{ Success rate of benchmarking DiffOG with Leto \cite{xu2024leto}.}
\label{tab:leto_comp}
\end{table}

\subsection{ {Inference Time-varying Constraint Experiment}}

\begin{table*}[h]
\centering
 {
\begin{tabular}{|c|ccccc|cc|}
\hline
       & \multicolumn{5}{c|}{Varying Constraints 0.05--0.3}                                                                                          & \multicolumn{2}{c|}{Static Constraints 0.1} \\ \cline{2-8} 
       & \multicolumn{1}{c|}{P1 Acc.} & \multicolumn{1}{c|}{P2 Acc.} & \multicolumn{1}{c|}{P3 Acc.} & \multicolumn{1}{c|}{P4 Acc.} & SR   & \multicolumn{1}{c|}{Acc.}     & SR      \\ \hline
Can    & 0.05/0.013                   & 0.08/0.022                   & 0.13/0.034                   & 0.22/0.043                   & 0.98 & 0.09/0.026                    & 0.98    \\
Square & 0.05/0.013                   & 0.07/0.018                   & 0.09/0.023                   & 0.17/0.028                   & 0.83 & 0.09/0.020                    & 0.87    \\ \hline
\end{tabular}}
\caption{ {Inference time-varying constraint experiment (Phases P1–P4 with bounds 0.05, 0.1, 0.2, and 0.3, respectively). The policy is trained on a fixed normalized constraint bound of 0.1. During inference, tighter bounds result in smoother motion, reflected by lower acceleration standard deviation and capped maximum acceleration, while preserving a success rate comparable to that of the policy using static constraints 0.1. Acc. represents acceleration and SR represents success rate. Accelerations here refer to maximum/standard deviation.}}
\label{tab:phase}
\end{table*}

\begin{figure}[h]
\centering
\begin{overpic}[trim=0 0 0 0,clip, width=0.48\textwidth]{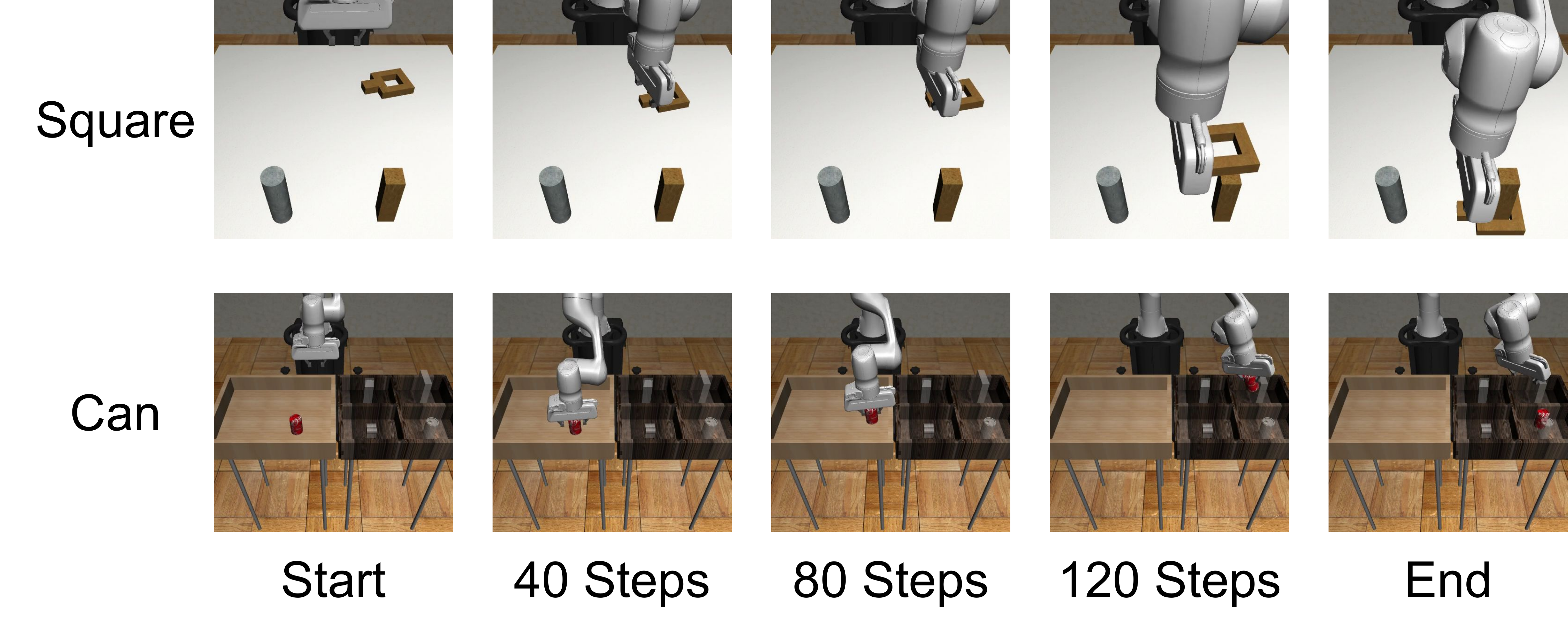}
\end{overpic}
\caption{ {Task phase illustration in the inference time-varying constraint experiment.}}
\label{fig:phase}
\end{figure}

 {In this experiment, the policy is trained on a fixed normalized constraint bound of 0.1 for the can and square tasks. However, during inference, we apply time-varying bounds: 0–40 steps use 0.05 (Phase 1), 40–80 use 0.1 (P2), 80–120 use 0.2 (P3), and 120 to task completion use 0.3 (P4). The task phases used for applying varying constraints during inference are illustrated in Fig.~\ref{fig:phase}. Note that the phases are determined using uniform intervals of 40 steps, without any task-specific selection, as our goal is to evaluate the low-level capability of handling varying constraints. As shown in Table~\ref{tab:phase}, despite being trained only with a static constraint, DiffOG can generalize zero-shot to time-varying constraints at inference time due to its interpretability and generalizability.}

\subsection{ {Residual Policy Baseline with Smoothness Penalty}}

\begin{figure*}[t]
\centering
\begin{overpic}[trim=0 0 0 0,clip, width=0.98\textwidth]{images/penalty.png}
\end{overpic}
\caption{ {Residual policy training with varying smoothness penalty weight $\alpha$.}}
\label{fig:penalty}
\end{figure*}

 {In this section, we design a residual policy training baseline that allows for smoothness regularization. Specifically, given the base action $\mathbf{a}_t$, we train a lightweight MLP to predict a residual action $\Delta \mathbf{a}_t = \mathrm{MLP}(\mathbf{a}_t)$, and define the final action as $\mathbf{y}_t = \mathbf{a}_t + \Delta \mathbf{a}_t$. The training loss is
\[
\mathcal{L}_p = \mathbb{E}_{\mathbf{a}_t \in \mathcal{D}}\left[\| \mathbf{y}_t  - \mathbf{a}_t\|^2 + \frac{\alpha}{2} \left(\frac{d\hat{\mathbf{c}}_t}{dt}\right)^{\!\!\mathrm{T}} \left(\frac{d\hat{\mathbf{c}}_t}{dt}\right)\right],
\]
where the second term is identical to the smoothness penalty used in DiffOG and encourages low velocity/acceleration over the resulting trajectory. The smoother the motion, the smaller this term becomes. Please refer to equations \eqref{eq:dc_y} and \eqref{eq:opt1} for implementation details. We conduct extensive ablation experiments by varying $\alpha$ across a wide range (from 0.01 to 16) and evaluate on three tasks: lift, can, and square. As shown in Fig.~\ref{fig:penalty}, we observe the following: {1)} when $\alpha$ is small (e.g., 0.01), the smoothness penalty is relatively weak, and the residual policy maintains comparable success rates to both DiffOG and the base policy; {2)} under the same setting, the residual reduces max and std acceleration over the base policy, indicating that smoother actions are indeed being generated; {3)} however, the improvement in smoothness still falls significantly short of DiffOG, with noticeably higher max and std acceleration of residual policy baseline; {4)} as $\alpha$ increases, task success rate consistently drops across all three tasks; {5)} although larger $\alpha$ generally leads to smoother motions, the gains quickly saturate, and the resulting trajectories remain less smooth than those produced by DiffOG, even when the success rate degrades substantially (e.g., $\alpha > 4$); and {6)} in some cases, further increasing $\alpha$ even slightly increase the max or std acceleration. These results indicate that DiffOG is a more effective solution, achieving substantial trajectory improvement without sacrificing policy performance.}

 {In addition to its empirical advantages, DiffOG offers a fundamental distinction from the residual policy with smoothness regularization. The residual policy cannot guarantee hard constraint satisfaction at inference time and lacks a structured, interpretable optimization framework. In contrast, DiffOG leverages differentiable trajectory optimization to explicitly enforce hard constraints and systematically refine trajectories. The structure of DiffOG leads to better alignment with task requirements and stronger overall performance, making DiffOG a more principled and generalizable method.}

\section{Limitations and Future Work}

Since this paper focuses on supervised policy learning, both training paradigms of DiffOG, dataset training and refine training, depend on pre-collected demonstrations. This reliance limits DiffOG’s ability to generalize to diverse trajectories beyond those in the dataset. Moreover, if the pre-collected dataset contains extremely jerky trajectories, standard supervised learning methods may struggle to substantially optimize these trajectories while preserving fidelity to the original dataset. A potential future direction involves leveraging reinforcement learning, which could enable finetuning of DiffOG on top of the existing dataset. This strategy would allow the model to incorporate multiple objectives, such as trajectory constraints and smoothness, without requiring extensive reward engineering. By simultaneously fine-tuning both the pre-trained policy and DiffOG itself, reinforcement learning could also facilitate broader exploration of trajectory diversity, ultimately enhancing DiffOG’s overall effectiveness.

 {Moreover, our method has demonstrated strong empirical performance under widely adopted action representations in imitation learning, such as joint angles \cite{zhao2023learning,wu2024gello} and 6D end-effector velocities \cite{mandlekar2022matters, wang2024poco, zhu2023viola}. Beyond these settings, we identify several promising directions for future work: 1) Certain action spaces that are less common in current imitation learning benchmarks, such as joint torques, also conform to the formulations of DiffOG. While not yet widely used, joint torques hold strong potential for force-based imitation learning, as they provide direct control over physical interactions. Extending our method to torque-based action spaces could enable more contact-rich manipulation. 2) Future extensions of the framework may address more complex representations, such as absolute $SO(3)$ rotations. In such cases, the assumption of simple discrete-time derivatives may no longer hold, requiring more expressive formulations of the underlying optimization problem.}

 {In addition, a natural extension of this work is to incorporate collision avoidance constraints into the DiffOG framework. This would require integrating environment and geometry information to formulate differentiable constraints that ensure safe interaction with surrounding objects. While this introduces additional complexity in constraint modeling, it would significantly broaden the applicability of DiffOG to tasks involving cluttered or dynamic environments.}

 {Lastly, DiffOG demonstrates the ability to perform inference with time-varying constraints, even when trained only with static ones. If a high-level semantic reasoning module is available to generate time-varying constraints, it could be seamlessly integrated with the DiffOG formulation to form a hierarchical structure. This would enable more phase-specific optimization tailored to different stages of a task, representing a promising direction for future work.}
\section{Conclusion}

We introduced DiffOG, a learning-based trajectory optimization framework that combines the strengths of neural networks with an interpretable differentiable optimization layer. By integrating trajectory smoothness, adherence to hard constraints, and fidelity to demonstrations into the training objectives, DiffOG enables visuomotor policies to produce smoother and safer action trajectories while maintaining performance. Through extensive evaluations on 11 simulated tasks and 2 real-world tasks, we demonstrated that DiffOG consistently outperforms existing trajectory processing baselines, such as greedy constraint clipping and penalty-based trajectory optimization, as well as constrained visuomotor policy Leto \cite{xu2024leto}. Our results demonstrate the generalization capability of DiffOG across diverse tasks.

\bibliographystyle{plainnat}
\bibliography{references}

@inproceedings{amos2017optnet,
  title={{OptNet}: Differentiable optimization as a layer in neural networks},
  author={Amos, Brandon and Kolter, J Zico},
  booktitle={Proc. 34th Int. Conf. Mach. Learn.},
  pages={136--145},
  year={2017}
}

@inproceedings{mandlekar2022matters,
  title={What Matters in Learning from Offline Human Demonstrations for Robot Manipulation},
  author={Mandlekar, Ajay and Xu, Danfei and Wong, Josiah and Nasiriany, Soroush and Wang, Chen and Kulkarni, Rohun and Fei-Fei, Li and Savarese, Silvio and Zhu, Yuke and Mart{\'\i}n-Mart{\'\i}n, Roberto},
  booktitle={Proc. Conf. Robot Learn.},
  pages={1678--1690},
  year={2022},
}

@article{schulman2014motion,
  title={Motion planning with sequential convex optimization and convex collision checking},
  author={Schulman, John and Duan, Yan and Ho, Jonathan and Lee, Alex and Awwal, Ibrahim and Bradlow, Henry and Pan, Jia and Patil, Sachin and Goldberg, Ken and Abbeel, Pieter},
  journal={The International Journal of Robotics Research},
  year={2014},
}

@article{pomerleau1988alvinn,
  title={Alvinn: An autonomous land vehicle in a neural network},
  author={Pomerleau, Dean A},
  journal={Advances in neural information processing systems},
  year={1988}
}

@INPROCEEDINGS{zhao2023learning,
    AUTHOR    = {Tony Z. Zhao AND Vikash Kumar AND Sergey Levine AND Chelsea Finn}, 
    TITLE     = {{Learning Fine-Grained Bimanual Manipulation with Low-Cost Hardware}}, 
    BOOKTITLE = {Proceedings of Robotics: Science and Systems}, 
    YEAR      = {2023}, 
    DOI       = {10.15607/RSS.2023.XIX.016} 
}

@inproceedings{florence2022implicit,
  title={Implicit behavioral cloning},
  author={Florence, Pete and Lynch, Corey and Zeng, Andy and Ramirez, Oscar A and Wahid, Ayzaan and Downs, Laura and Wong, Adrian and Lee, Johnny and Mordatch, Igor and Tompson, Jonathan},
  booktitle={Proceedings of Conference on Robot Learning},
  year={2022},
}

@inproceedings{chi2023diffusionpolicy,
	title={Diffusion Policy: Visuomotor Policy Learning via Action Diffusion},
	author={Chi, Cheng and Feng, Siyuan and Du, Yilun and Xu, Zhenjia and Cousineau, Eric and Burchfiel, Benjamin and Song, Shuran},
	booktitle={Proceedings of Robotics: Science and Systems},
	year={2023}
}

@article{xiao2021barriernet,
  author={Xiao, Wei and Wang, Tsun-Hsuan and Hasani, Ramin and Chahine, Makram and Amini, Alexander and Li, Xiao and Rus, Daniela},
  journal={IEEE Transactions on Robotics}, 
  title={BarrierNet: Differentiable Control Barrier Functions for Learning of Safe Robot Control}, 
  year={2023},
  doi={10.1109/TRO.2023.3249564}}

@inproceedings{
retchin2023koopman,
title={Koopman Constrained Policy Optimization: A Koopman operator theoretic method for differentiable optimal control in robotics},
author={Matthew Retchin and Brandon Amos and Steven Brunton and Shuran Song},
booktitle={ICML 2023 Workshop on Differentiable Almost Everything: Differentiable Relaxations, Algorithms, Operators, and Simulators},
year={2023},
url={https://openreview.net/forum?id=3W7vPqWCeM}
}

@article{amos2018differentiable,
  title={Differentiable {MPC} for end-to-end planning and control},
  author={Amos, Brandon and Jimenez, Ivan and Sacks, Jacob and Boots, Byron and Kolter, J Zico},
  journal={Advances in Neural Information Processing Systems},
  year={2018}
}

@article{ratliff2018riemannian,
  title={Riemannian motion policies},
  author={Ratliff, Nathan D and Issac, Jan and Kappler, Daniel and Birchfield, Stan and Fox, Dieter},
  journal={arXiv preprint arXiv:1801.02854},
  year={2018}
}

@inproceedings{cheng2020rmp,
  title={Rmp flow: A computational graph for automatic motion policy generation},
  author={Cheng, Ching-An and Mukadam, Mustafa and Issac, Jan and Birchfield, Stan and Fox, Dieter and Boots, Byron and Ratliff, Nathan},
  booktitle={Algorithmic Foundations of Robotics XIII: Proceedings of the 13th Workshop on the Algorithmic Foundations of Robotics 13},
  pages={441--457},
  year={2020},
  organization={Springer}
}

@inproceedings{li2021rmp2,
  title={{RMP2}: A structured composable policy class for robot learning},
  author={Li, Anqi and Cheng, Ching-An and Rana, M Asif and Xie, Man and Van Wyk, Karl and Ratliff, Nathan and Boots, Byron},
  booktitle={Proceedings of Robotics: Science and Systems},
  year={2021}
}

@article{
wan2024difftori,
title={Diff{TORI}: Differentiable Trajectory Optimization for Deep Reinforcement and Imitation Learning},
author={Weikang Wan and Ziyu Wang and Yufei Wang and Zackory Erickson and David Held},
journal={Advances in Neural Information Processing Systems},
year={2024},
url={https://openreview.net/forum?id=Mwj57TcHWX}
}

@article{xu2024letac,
  author={Xu, Zhengtong and She, Yu},
  journal={IEEE Transactions on Robotics}, 
  title={{LeTac-MPC}: Learning Model Predictive Control for Tactile-Reactive Grasping}, 
  year={2024},
  doi={10.1109/TRO.2024.3463470}}

@inproceedings{chi2024universal,
	title={Universal Manipulation Interface: In-The-Wild Robot Teaching Without In-The-Wild Robots},
	author={Chi, Cheng and Xu, Zhenjia and Pan, Chuer and Cousineau, Eric and Burchfiel, Benjamin and Feng, Siyuan and Tedrake, Russ and Song, Shuran},
	booktitle={Proceedings of Robotics: Science and Systems},
	year={2024}
}

@inproceedings{fu2024mobile,
  author    = {Fu, Zipeng and Zhao, Tony Z. and Finn, Chelsea},
  title     = {Mobile ALOHA: Learning Bimanual Mobile Manipulation with Low-Cost Whole-Body Teleoperation},
  booktitle = {Proceedings of Conference on Robot Learning},
  year      = {2024},
}

@inproceedings{wang2024poco,
  title={PoCo: Policy Composition from and for Heterogeneous Robot Learning},
  author={Wang, Lirui and Zhao, Jialiang and Du, Yilun and Adelson, Edward H and Tedrake, Russ},
	booktitle={Proceedings of Robotics: Science and Systems},
	year={2024}
}

@inproceedings{lee2024behavior,
  title={Behavior Generation with Latent Actions},
  author={Lee, Seungjae and Wang, Yibin and Etukuru, Haritheja and Kim, H Jin and Shafiullah, Nur Muhammad Mahi and Pinto, Lerrel},
  booktitle={Proceedings of International Conference on Machine Learning},
 year = {2024}
}

@inproceedings{yu2020meta,
  title={Meta-world: A benchmark and evaluation for multi-task and meta reinforcement learning},
  author={Yu, Tianhe and Quillen, Deirdre and He, Zhanpeng and Julian, Ryan and Hausman, Karol and Finn, Chelsea and Levine, Sergey},
  booktitle={Proceedings of Conference on robot learning},
  year={2020},
}

@inproceedings{wu2024gello,
  title={Gello: A general, low-cost, and intuitive teleoperation framework for robot manipulators},
  author={Wu, Philipp and Shentu, Yide and Yi, Zhongke and Lin, Xingyu and Abbeel, Pieter},
  booktitle={Proceedings of IEEE/RSJ International Conference on Intelligent Robots and Systems},
  year={2024},
}

@inproceedings{zhu2023viola,
  title={Viola: Imitation learning for vision-based manipulation with object proposal priors},
  author={Zhu, Yifeng and Joshi, Abhishek and Stone, Peter and Zhu, Yuke},
  booktitle={Proceedings of Conference on Robot Learning},
  year={2023},
}

@article{xu2024leto,
author = {Xu, Zhengtong and She, Yu},
year = {2024},
title = {LeTO: Learning Constrained Visuomotor Policy With Differentiable Trajectory Optimization},
journal = {IEEE Transactions on Automation Science and Engineering},
doi = {10.1109/TASE.2024.3486542}
}

@inproceedings{octo_2023,
    title={Octo: An Open-Source Generalist Robot Policy},
    author = {{Octo Model Team} and Dibya Ghosh and Homer Walke and Karl Pertsch and Kevin Black and Oier Mees and Sudeep Dasari and Joey Hejna and Charles Xu and Jianlan Luo and Tobias Kreiman and {You Liang} Tan and Lawrence Yunliang Chen and Pannag Sanketi and Quan Vuong and Ted Xiao and Dorsa Sadigh and Chelsea Finn and Sergey Levine},
    booktitle = {Proceedings of Robotics: Science and Systems},
    year = {2024},
}

@article{o2023open,
  title={Open {X-E}mbodiment: Robotic Learning Datasets and {RT-X} Models},
  author={O'Neill, Abby and Rehman, Abdul and Gupta, Abhinav and Maddukuri, Abhiram and Gupta, Abhishek and Padalkar, Abhishek and Lee, Abraham and Pooley, Acorn and Gupta, Agrim and Mandlekar, Ajay and others},
  journal={arXiv preprint arXiv:2310.08864},
  year={2023}
}

@article{black2024pi_0,
  title={$\pi_0 $: A Vision-Language-Action Flow Model for General Robot Control},
  author={Black, Kevin and Brown, Noah and Driess, Danny and Esmail, Adnan and Equi, Michael and Finn, Chelsea and Fusai, Niccolo and Groom, Lachy and Hausman, Karol and Ichter, Brian and others},
  journal={arXiv preprint arXiv:2410.24164},
  year={2024}
}

@article{shafiullah2023bringing,
  title={On bringing robots home},
  author={Shafiullah, Nur Muhammad Mahi and Rai, Anant and Etukuru, Haritheja and Liu, Yiqian and Misra, Ishan and Chintala, Soumith and Pinto, Lerrel},
  journal={arXiv preprint arXiv:2311.16098},
  year={2023}
}

@article{seo2024legato,
  title={LEGATO: Cross-Embodiment Imitation Using a Grasping Tool},
  author={Seo, Mingyo and Park, H Andy and Yuan, Shenli and Zhu, Yuke and Sentis, Luis},
  journal={arXiv preprint arXiv:2411.03682},
  year={2024}
}

@inproceedings{xu2024flow,
  title={Flow as the cross-domain manipulation interface},
  author={Xu, Mengda and Xu, Zhenjia and Xu, Yinghao and Chi, Cheng and Wetzstein, Gordon and Veloso, Manuela and Song, Shuran},
  year = {2024},
  booktitle={Proceedings of Conference on Robot Learning}
}

@inproceedings{wen2023any,
  title={Any-point trajectory modeling for policy learning},
  author={Wen*, Chuan and Lin*, Xingyu and So*, John and Chen, Kai and Dou, Qi and Gao, Yang and Abbeel, Pieter},
  booktitle={Proceedings of Robotics: Science and Systems},
  year={2024}
}

@inproceedings{wangequivariant,
  title={Equivariant Diffusion Policy},
  author={Wang, Dian and Hart, Stephen and Surovik, David and Kelestemur, Tarik and Huang, Haojie and Zhao, Haibo and Yeatman, Mark and Wang, Jiuguang and Walters, Robin and Platt, Robert},
  year = {2024},
  booktitle={Proceedings of Conference on Robot Learning}
}

@inproceedings{wang2024dexcap,
  title = {DexCap: Scalable and Portable Mocap Data Collection System for Dexterous Manipulation},
  author = {Wang, Chen and Shi, Haochen and Wang, Weizhuo and Zhang, Ruohan and Fei-Fei, Li and Liu, C. Karen},
  booktitle={Proceedings of Robotics: Science and Systems},
  year={2024}
}

@article{liu2024rdt,
  title={Rdt-1b: a diffusion foundation model for bimanual manipulation},
  author={Liu, Songming and Wu, Lingxuan and Li, Bangguo and Tan, Hengkai and Chen, Huayu and Wang, Zhengyi and Xu, Ke and Su, Hang and Zhu, Jun},
  journal={arXiv preprint arXiv:2410.07864},
  year={2024}
}

@inproceedings{rana2021towards,
  title={Towards coordinated robot motions: End-to-end learning of motion policies on transform trees},
  author={Rana, M Asif and Li, Anqi and Fox, Dieter and Chernova, Sonia and Boots, Byron and Ratliff, Nathan},
  booktitle={Proceedings of IEEE/RSJ International Conference on Intelligent Robots and Systems},
  year={2021},
}

@inproceedings{ze20243d,
  title={{3D Diffusion Policy:} Generalizable Visuomotor Policy Learning via Simple {3D} Representations},
  author={Ze, Yanjie and Zhang, Gu and Zhang, Kangning and Hu, Chenyuan and Wang, Muhan and Xu, Huazhe},
  booktitle={Proceedings of Robotics: Science and Systems},
  year={2024}
}

@article{
song2021denoising,
title={Denoising Diffusion Implicit Models},
author={Jiaming Song and Chenlin Meng and Stefano Ermon},
journal={Proceedings of International Conference on Learning Representations},
year={2021},
url={https://openreview.net/forum?id=St1giarCHLP}
}

@article{ho2020denoising,
  title={Denoising diffusion probabilistic models},
  author={Ho, Jonathan and Jain, Ajay and Abbeel, Pieter},
  journal={Advances in Neural Information Processing Systems},
  year={2020}
}

\section*{Appendix}

\subsection{Baselines}

In this section, we introduce two trajectory processing baseline methods: constraint clipping and penalty-based trajectory optimization. These baselines apply post-hoc processing to action trajectories generated by policies. Since visuomotor policies lack mechanisms for enforcing constraints and adjusting actions, model-based post-hoc trajectory processing methods are widely adopted in real-world systems.

The goal of both constraint clipping and penalty-based trajectory optimization is to regulate the action outputs of base policies, ensuring adherence to predefined constraints. These methods are formulated within model-based optimization frameworks. Constraint clipping can be viewed as a greedy optimization approach, whereas penalty-based optimization performs global optimization across all time steps by leveraging penalties.

According to equation \eqref{eq:select}, for the sake of simplicity, we only include the action dimensions to be considered for trajectory optimization in the formulas presented in this section. For action dimensions that do not require trajectory optimization, such as grasping actions, the corresponding action dimensions from the base policy output can be retained. These can then be concatenated with the optimized trajectory actions after the trajectory optimization process. In this case, the input of the optimization problem becomes
\begin{align*}
{\mathbf{c}}_t = \mathbf{S}{\mathbf{a}}_t = \left[
{c}_t^{\mathrm{T}},{c}_{t+1}^{\mathrm{T}},\ldots,{c}_{t+T_p-1}^{\mathrm{T}}
\right]^{\mathrm{T}} \in \mathbb{R}^{T_pD_c},
\end{align*}
where $\mathbf{S}$ is for  {selecting the action dimensions} to be considered for trajectory optimization. Similarly, we define an output variable
\begin{align*}
{\mathbf{\hat{c}}}_t  = \left[
\hat{c}_t^{\mathrm{T}},\hat{c}_{t+1}^{\mathrm{T}},\ldots,\hat{c}_{t+T_p-1}^{\mathrm{T}}
\right]^{\mathrm{T}} = \mathbf{S}{\mathbf{y}}_t \in \mathbb{R}^{T_pD_c},
\end{align*}
and we denote ${\mathbf{\hat{c}}}_t^{\star}$ as the optimized output variable. For continuous trajectory optimization, like what we do in Section~\ref{sec:inference}, a previous action $\hat{c}_{t-1}$ from last prediction is also needed as input.

Based on equation \eqref{eq:vel_con}, the optimization is subject to constraints defined by two vectors, ${d}_{\text{min}} \Delta t$ and ${d}_{\text{max}} \Delta t$, which specify the minimum and maximum allowable bounds for the change in each dimension of the action vector. 

\subsubsection{Greedy Constraint Clipping}
The goal of constraint clipping is to compute an optimized sequence of  actions $\hat{\mathbf{c}}_t$ that satisfies the given constraints by greedy strategy. 

For the initial time step, the optimized action is initialized as 
$\hat{c}^{\star}_{t-1} =\hat{c}_{t-1}.$
Then, for each subsequent time step $t+k, k = 0 \ldots T_p-1$, the difference between the input action $c_{t+k}$ and the optimized action from the previous time step $\hat{c}_{t+k-1}^{\star}$ is calculated as $\Delta c_{t+k} = c_{t+k} - \hat{c}_{t+k-1}^{\star}$. This difference $\Delta c_{t+k}$ is then clipped element-wise to ensure it lies within the bounds ${d}_{\text{min}}\Delta t \leq \Delta c_{t+k} \leq {d}_{\text{max}} \Delta t$, resulting in 
$$\Delta c_{t+k}^{\text{clip}} = \text{clip}(\Delta c_{t+k}, {d}_{\text{min}}\Delta t, {d}_{\text{max}}\Delta t).$$
The optimized action at time $t+k$ is then updated as 
$$\hat{c}_{t+k}^{\star} = \hat{c}_{t+k-1}^{\star} + \Delta c_{t+k}^{\text{clip}}.$$

Through this process, the optimized trajectory ${\mathbf{\hat{c}}}_t^{\star}$ satisfies the constraints at each time step while remaining as close as possible to the input trajectory ${\mathbf{c}}_t $ by greedy optimization.

\subsubsection{Penalty-Based Trajectory Optimization}

The goal of penalty-based optimization is to compute an optimized sequence of actions $\hat{\mathbf{c}}_t$ that satisfies the given constraints by leveraging penalties to perform global optimization across all time steps. Unlike constraint clipping, which operates greedily on each time step, penalty-based optimization considers all time steps jointly.

Denote 
$
{\mathbf{c}}_{t-1} = \left[\hat{c}_{t-1}^{\mathrm{T}},
{c}_t^{\mathrm{T}},{c}_{t+1}^{\mathrm{T}},\ldots,{c}_{t+T_p-2}^{\mathrm{T}}
\right]^{\mathrm{T}} \in \mathbb{R}^{T_pD_c}.
$
Then define the difference sequence $\Delta\mathbf{c}_{t} = {\mathbf{c}}_{t} - {\mathbf{c}}_{t-1}.$ At each iteration, the violations of the constraints are computed as:
\[
\mathbf{v}_{\text{max}} = \max(0, \Delta\mathbf{c}_{t} - \mathbf{d}_{\text{max}}\Delta t),
\]
\[
\mathbf{v}_{\text{min}} = \max(0, \mathbf{d}_{\text{min}}\Delta t - \Delta\mathbf{c}_{t}),
\]
where $\mathbf{d}_{\text{max}} = \left[
{d}_{\text{max}}^{\mathrm{T}},{d}_{\text{max}}^{\mathrm{T}},\ldots,{d}_{\text{max}}^{\mathrm{T}}
\right]^{\mathrm{T}} \in \mathbb{R}^{T_pD_c}$ is a sequence of upper constraint bounds, and we can also define the sequence of lower constraint bounds $\mathbf{d}_{\text{min}}$ following this format. Here, the max and min operations are applied element-wise to the vector, meaning that each element of the vector is individually processed through the max or min operation. The penalty for the global optimization is then expressed as:
\[
\nabla = \mathbf{v}_{\text{max}} - \mathbf{v}_{\text{min}}.
\]
We can initialize ${\mathbf{\hat{c}}}_t = {\mathbf{{c}}}_t$ for warm start.
The optimized trajectory ${\mathbf{\hat{c}}}_t$ is updated using penalty:
\[
{\mathbf{\hat{c}}}_t \leftarrow {\mathbf{\hat{c}}}_t- \eta \nabla,
\]
where $\eta$ is the step size. This process is repeated for a fixed number of iterations or until the maximum violation across all time steps falls below a given tolerance $\epsilon^\mathbf{v}$, denoted as
$
\|\mathbf{v}_{\text{max}}\| + \|\mathbf{v}_{\text{min}}\| < \epsilon^\mathbf{v}.
$
Then we can get the optimized trajectory ${\mathbf{\hat{c}}}_t^{\star}$.
Through this approach,  ${\mathbf{\hat{c}}}_t^{\star}$ satisfies the constraints while remaining as close as possible to the original input trajectory $ {\mathbf{{c}}}_t$ by penalty-based optimization.

\subsection{Action Space Assumption}
\begin{assumption}
The proposed method assumes that the discrete-time derivative $\frac{d{{c}}_t}{dt}$ in \eqref{eq:dc} represents explicit, element‑wise accelerations or velocities.
\label{assu:action}
\end{assumption}

The purpose of introducing Assumption \ref{assu:action} is to justify the addition of constraints in subsequent sections. The fundamental question addressed here is whether imposing constraints on $\frac{d{{c}}_t}{dt}$ can physically restrict the actions generated by the policy. When Assumption \ref{assu:action} holds, the answer is straightforward: the constraints directly regulate explicit physical quantities such as velocities or accelerations. However, even when Assumption \ref{assu:action} does not strictly hold, certain forms of constraints may still provide meaningful regulation over the robot’s motion. We elaborate on these cases in the following discussion.

Among commonly used action spaces for robot policies, Assumption \ref{assu:action} holds under the following conditions:

1. The action space consists of joint angles, where the derivative corresponds to joint velocities. This type of action space is commonly used in systems that collect data through leader-follower arms, such as ALOHA \cite{zhao2023learning} and GELLO \cite{wu2024gello}.

2. The action space consists of the robot end-effector linear position in $xyz$, where the derivative corresponds to linear velocity. This action space is used in the push-T task in \cite{chi2023diffusionpolicy} and the Meta-World benchmark \cite{yu2020meta}.

Another widely adopted action space is 
$
[v_x,\, v_y,\, v_z,\,
\dot r_x,\, \dot r_y,\, \dot r_z]^{\mathrm T}
\in \mathbb R^{6},
$
where $\mathbf v_{xyz}=[v_x,v_y,v_z]^{\mathrm T}$ is translational velocity, and  $\dot{\mathbf r}=[\dot r_x,\dot r_y,\dot r_z]^{\mathrm T}$ is the derivative of the axis-angle rotation vector~\cite{mandlekar2022matters, wang2024poco, zhu2023viola, xu2024leto}. Its time derivative,
$
[
\dot v_x,\, \dot v_y,\, \dot v_z,\,
\ddot r_x,\, \ddot r_y,\, \ddot r_z
]^{\!\mathrm{T}},
$
includes linear accelerations and axis-angle second derivatives, which differ from the true angular acceleration $\dot{\boldsymbol\omega}$. 

Although $\ddot{\mathbf{r}}$ does not directly represent $\dot{\boldsymbol{\omega}}$, constraining $\ddot{\mathbf{r}}$ effectively regulates $\dot{\boldsymbol{\omega}}$. The angular acceleration is related by $\dot{\boldsymbol{\omega}} = \mathbf{J}(\mathbf{r}) \ddot{\mathbf{r}} + \dot{\mathbf{J}}(\mathbf{r}) \dot{\mathbf{r}}$, where $\mathbf{J}(\mathbf{r})$ is the Jacobian matrix mapping axis-angle velocity to angular velocity. In practice, $\mathbf{J}(\mathbf{r})$ typically remains bounded, as evidenced by empirical successes of such action space in imitation learning frameworks~\cite{mandlekar2022matters, wang2024poco, zhu2023viola}, which achieve stable rotations even without explicit constraints. Different from these prior work, DiffOG explicitly and effectively constrains rotational motion, resulting in smoother and more controlled behaviors. Extensive experimental results demonstrate the effectiveness of DiffOG.

\subsection{Task Descriptions and Experimental Setup}\label{sec:task}

For the end-effector 6-DOF velocity action space, we include five tasks: lift, can, square, tool hang, and transport, all sourced from Robomimic benchmark \cite{mandlekar2022matters}. For the end-effector linear position action space, we include six tasks: push-T from \cite{chi2023diffusionpolicy}, along with pick place wall, shelf place, disassemble, stick push, and stick pull from Meta-World \cite{yu2020meta}. Notably, we selected all the ``very hard" tasks identified in Meta-World as marked in \cite{ze20243d}. Lastly, for the joint angle action space, we include two real-world tasks, arrange desk and move the stack, with data collection and policy experiments conducted using ALOHA \cite{zhao2023learning}.

\begin{table*}[t]
\centering
\begin{tabular}{|c|c|c|c|c|c|c|c|}
\hline
\textbf{Task}            & \textbf{Bound} & \textbf{Unnormalized Bound} & \textbf{Demos} & \textbf{Action Space}                 & \bm{$\Delta t$} & \bm{$\alpha$} & \textbf{High Dim. Obs.}\\ \hline
Lift            & 0.1    & 0.1                                               & 200   & 6-DOF Vel. + 1-DOF Grasping  &           0.05    s         & 4   & $2\times3\times84\times84$  \\ \hline
Can             & 0.1    & 0.1                                               & 200   & 6-DOF Vel. + 1-DOF Grasping  &            0.05     s        & 4  &$2\times3\times84\times84$     \\ \hline
Square          & 0.1    & 0.1                                               & 200   & 6-DOF Vel. + 1-DOF Grasping  &            0.05    s        & 4    &$2\times3\times84\times84$    \\ \hline
Tool Hang       & 0.1    & 0.1                                               & 200   & 6-DOF Vel. + 1-DOF Grasping  &             0.05    s        & 4   &$2\times3\times240\times240$    \\ \hline
Transport       & 0.1    & 0.1                                               & 200   & 12-DOF Vel. + 2-DOF Grasping &             0.05   s        & 4     &$4\times3\times84\times84$  \\ \hline
Push-T          & 0.05   & {[}12.5,12.2{]}  Pixels    & 90    & 2-DOF Position               &             0.1  s          & 1    & $1\times3\times96\times96$  \\ \hline
Pick Place Wall & 0.6    & {[}0.53, 0.92, 3.75{]}  mm & 30    & 3-DOF Pos. + 1-DOF Grasping  &           0.0125 s               & 4  & $3\times512$    \\ \hline
Shelf Place     & 0.6    & {[}2.32, 1.47, 4.13{]}  mm & 30    & 3-DOF Pos. + 1-DOF Grasping  &           0.0125 s              & 4   & $3\times512$    \\ \hline
Disassemble     & 0.6    & {[}0.65, 1.10, 0.61{]}  mm & 30    & 3-DOF Pos. + 1-DOF Grasping  &           0.0125 s               & 4   & $3\times512$      \\ \hline
Stick Push      & 0.6    & {[}1.03, 0.15, 0.68{]}  mm & 30    & 3-DOF Pos. + 1-DOF Grasping  &           0.0125 s              & 4    & $3\times512$  \\ \hline
Stick Pull      & 0.6    & {[}4.46, 1.44, 2.49{]}  mm & 30    & 3-DOF Pos. + 1-DOF Grasping  &           0.0125 s               & 4    & $3\times512$   \\ \hline
Arrange Desk    & 0.1    &              \parbox[c]{2.8cm}{\centering \vspace{3pt}[0.070, 0.071, 0.102,\\ 0.030, 0.056, 0.148,\\ 0.078, 0.089, 0.105,\\ 0.039, 0.069, 0.120] rad\vspace{3pt} } & 100   & 12-DOF Angle + 2-DOF Grasping &           0.25    s          & 4      & $3\times3\times240\times320$   \\ \hline
Move the Stack  & 0.05   &      \parbox[c]{2.8cm}{\centering \vspace{3pt}[0.049, 0.051, 0.057,\\ 0.045, 0.038, 0.069] rad\vspace{2pt}}         & 100   & 6-DOF Angle + 1-DOF Grasping  &              $\frac{1}{6}$     s      & 4    & $2\times3\times240\times320$    \\ \hline
\end{tabular}
\caption{Summary of hyperparameters for trajectory optimization and dataset. $\alpha$ is the smoothing weight, $\frac{1}{\Delta t}$ represents the frequency of policy inference, and ``Demos” indicates the number of demonstration episodes. ``High Dim. Obs" refers to the dimensionality of high-dimensional observations. For instance, $2\times3\times84\times84$ corresponds to RGB observations from two cameras, each with a shape of  $3\times84\times84$, while $3\times512$ represents point cloud data with a shape of $3\times512$. ``Bound" refers to $d_{\text{max}} \Delta t $ in equation \eqref{eq:vel_con}. ``Unnormalized Bound" refers to constraint bounds unnormalized to dataset scale.}
\label{tab:dataset}
\end{table*}

The goal of the arrange desk task is to organize two cups and a bowl placed on the table. This is a dual-arm, long-horizon manipulation task. As shown in Fig.~\ref{fig:bowl_process}, the two arms first individually grasp a bowl and a cup, stacking them together. Next, one arm moves the stack to the location of the coaster. Finally, the other arm picks up the remaining cup and stacks it on top of the stack already placed on the coaster. 

The goal of the move the stack task is to pick up a glass cup with a spoon placed on top of it and move it to the target coaster, as shown in Fig.~\ref{fig:glass_process}. Since the spoon is balanced on the glass cup, the task requires smooth robotic motions; otherwise, the spoon may fall off. Therefore, this task cannot tolerate excessively jerky motions.

\begin{figure}[t]
\centering
\begin{overpic}[trim=0 0 0 0,clip, width=0.5\textwidth]{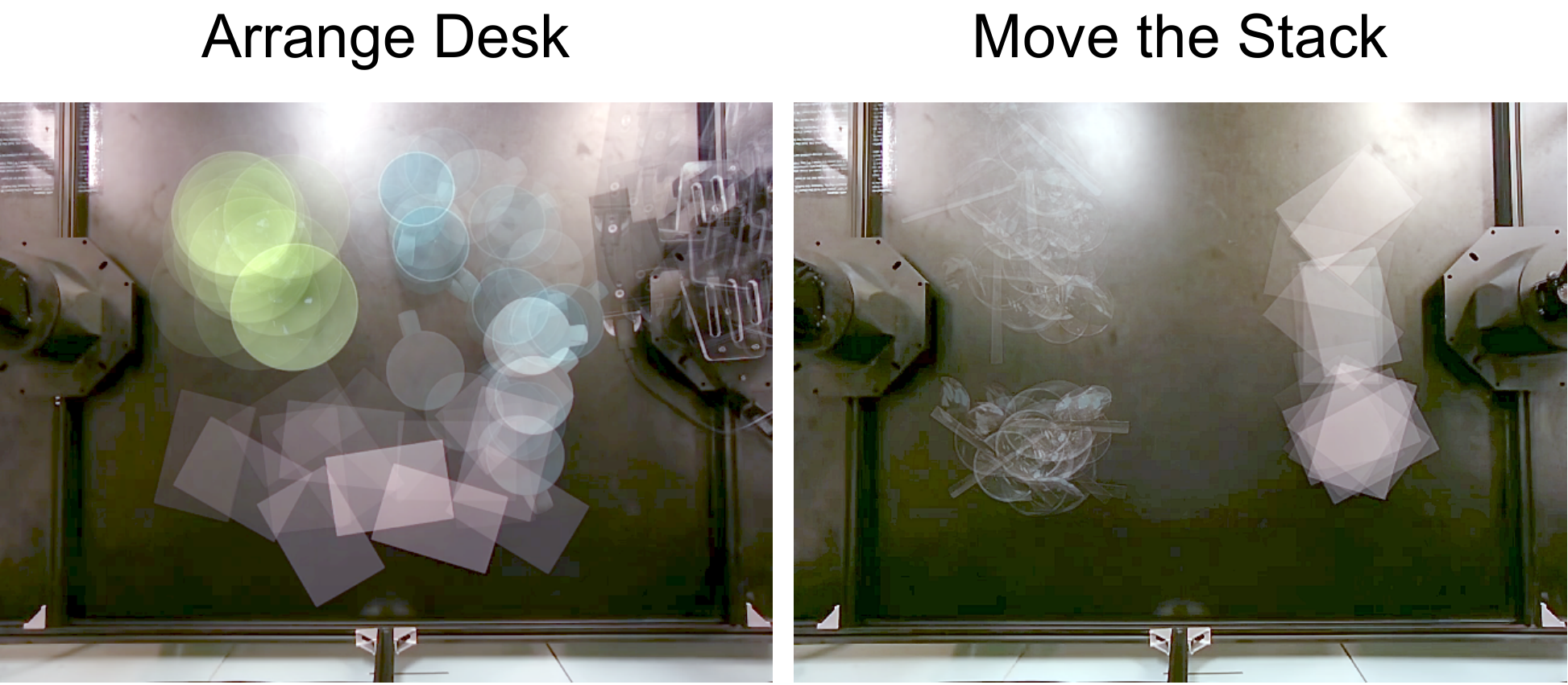}
\end{overpic}
\caption{ {Overlay of all 15 DiffOG trial layouts.}}
\label{fig:random}
\end{figure}

For real-world tasks, we randomized the initial object placements within a spatial range for each trial. An overlay visualization of all 15 DiffOG trials is shown in Fig.~\ref{fig:random}.

\subsection{ {Evaluating DiffOG under Multiple Motion Constraints}}

\begin{table*}[t]
\centering
 {
\begin{tabular}{|c|c|ccc|c|}
\hline
Constraint Type  & No Constraint & \multicolumn{3}{c|}{Acc.}   & Acc. + Vel. \\ \hline
                 & Base Policy  & Clipping & Penalty-Based & DiffOG  & DiffOG      \\ \hline
Success Rate     & \textbf{0.91}         & 0.78 & 0.76 & 0.87     & 0.87        \\ \cline{1-1}
Max Velocity     & 0.72         & 0.63 & 0.66 & 0.65    & \textbf{0.39}        \\ \cline{1-1}
Max Acceleration & 0.38         & \textbf{0.09}   & \textbf{0.09}   & \textbf{0.09}     & \textbf{0.09}        \\ \hline
\end{tabular}}
\caption{ {Performance on the square task under different constraint settings. DiffOG with acceleration and velocity constraints outperforms baselines that apply only acceleration constraints, highlighting its flexibility in handling multiple motion constraints.}}
\label{tab:acc_vel}
\end{table*}

 {To further demonstrate the flexibility of DiffOG in handling diverse motion-related constraints, we conducted an experiment on the square task involving multiple motion constraints, as shown in Table~\ref{tab:acc_vel}. In this new experiment, we additionally incorporated an explicit velocity constraint (with a normalized bound of 0.5) into optimization problem \eqref{eq:opt1}, thereby enforcing both velocity and acceleration limits during training and inference.  As shown in Table~\ref{tab:acc_vel}, while the addition of a velocity constraint makes the feasible solution space strictly smaller and increases the difficulty of the optimization problem, DiffOG still achieves higher success rates than baseline methods that are only subject to the acceleration constraint. This highlights DiffOG’s flexibility in integrating multiple motion constraints and its effectiveness in producing trajectories that are both well-constrained and task-successful.}

\subsection{Hyperparameters and Dataset}\label{sec:info}

The summary of hyperparameters for trajectory optimization and dataset as shown in Table~\ref{tab:dataset}.
In the experiments of this paper, we consistently adopt symmetric constraints which is $d_{\text{max}} = -d_{\text{min}}$. Since the actions require normalization to [-1,1] during training \cite{chi2023diffusionpolicy}, all constraints are applied on the normalized action scale for DiffOG. For simplicity, when applying constraints on the normalized action scale, the same value is used for each action dimension. Therefore, the ``Bound" column contains only a single value for each task. Since the distribution of action values differs across task datasets, the bounds for each action dimension may vary after unnormalization. For Robomimic tasks \cite{mandlekar2022matters}, the action values in the dataset are already normalized to the range of $[-1, 1]$. Therefore, the bound and the unnormalized bound are identical. 

We can observe that, except for Robomimic tasks, the physical values of the constraint bounds can be computed by dividing the ``Unnormalized Bounds" column by the $\Delta t$ column. For instance, in real tasks such as move the stack and arrange desk, this calculation yields the upper and lower limits of joint angular velocities (measured in rad/s).

The bound of 0.1 for Robomimic tasks represents normalized acceleration. To convert the it back to the raw form, use a scalar of 20 $\text{m}/\text{s}^2$ for linear acceleration and 200 $\text{rad}/\text{s}^2$ for second derivative of axis-angle \cite{mandlekar2022matters}. For example, the maximum linear acceleration is  $0.1 \times 20~\text{m}/\text{s}^2 = 2~\text{m}/\text{s}^2$.

Since constraint clipping and penalty-based trajectory optimization do not require training, constraints can be directly applied to unnormalized actions. These constraints are equivalent to those imposed by DiffOG but are applied directly as unnormalized constraints. The parameters for penalty-based trajectory optimization are listed in Table~\ref{tab:gra}. As mentioned earlier, the actions provided by the Robomimic dataset are already normalized to the range of $[-1, 1]$. The constraint bound we impose in this case is $0.1$. Based on this, the violation tolerance for penalty-based optimization is set to a relatively small value of $0.004$. For other tasks, the constraints are applied to raw action values, and the violation tolerances are determined based on the scale of unnormalized bounds provided in Table~\ref{tab:dataset}.

Finally, we provide the hyperparameters of the transformer used in DiffOG, as shown in Table~\ref{tab:transformer}.

\begin{table}[]
\begin{tabular}{|c|c|c|c|}
\hline
\textbf{Task}             & \textbf{Step Size} \bm{$\eta$} & \textbf{Tolerance} \bm{$\epsilon^\mathbf{v}$} & \textbf{Max Iterations} \\ \hline
Robomimic Tasks  & 0.03             & 0.004                           & 10000          \\ \hline
Push-T           & 0.1              & 0.1                             & 1000           \\ \hline
Meta-World Tasks & 0.1              & 0.05                            & 1000           \\ \hline
Real ALOHA tasks & 0.05             & 0.01                            & 1000           \\ \hline
\end{tabular}
\caption{Parameters of penalty-based trajectory optimization.}
\label{tab:gra}
\end{table}

\begin{table}[t]
\centering
\begin{tabular}{|c|c|c|c|c|}
\hline
\textbf{Dropout} & \textbf{Heads} & \textbf{Layers} & \textbf{Embed. Dim.} & \textbf{Feed. Dim.} \\ \hline
0.1     & 4          & 2           & 256         & 256              \\ \hline
\end{tabular}
\caption{Hyperparameters of transformer encoder.}
\label{tab:transformer}
\end{table}

\subsection{Computation of Trajectory Metrics}

In this section, we introduce details of computing trajectory metrics shown in Figs.~\ref{fig:traj} and \ref{fig:traj2}. 

The trajectory metrics shown in  Figs.~\ref{fig:traj} and \ref{fig:traj2} are calculated as the average over the final 10 checkpoints, with each checkpoint evaluated across 50 different environment initializations. The results are further averaged over 3 training seeds, totaling 150 environment evaluations.

For position control, we present the velocity metric; for velocity control, we present the acceleration metric. For joint angle control, we provide not only the velocity metric (first derivative) but also the acceleration metric (second derivative). For a single evaluation in a specific environment, the calculation of ``max" involves taking the maximum value of each action dimension across all time steps and then averaging these maximum values across all action dimensions. The calculation of ``standard deviation" involves computing the standard deviation of the metric for each action dimension over time and then averaging these standard deviations across all action dimensions. This calculation of the standard deviation reflects the degree of jerkiness in the trajectory over time steps.

The acceleration metrics for Robomimic tasks are normalized values, and are averaged on linear accelerations and second derivatives of axis-angle. Please refer to Section~\ref{sec:info} to see more information of unnormalizing them to their raw form, which are measured in $\text{m}/\text{s}^2$ and $\text{rad}/\text{s}^2$.

\end{document}